%% file: main.tex
\documentclass{article}

\usepackage[utf8]{inputenc} % allow utf-8 input
\usepackage[T1]{fontenc}    % use 8-bit T1 fonts
\usepackage{hyperref}       % hyperlinks
\usepackage{url}            % simple URL typesetting
\usepackage{booktabs}       % professional-quality tables
\usepackage{amsfonts}       % blackboard math symbols
\usepackage{nicefrac}       % compact symbols for 1/2, etc.
\usepackage{microtype}      % microtypography

\newif\ifdraft
\draftfalse

\newif\ifuglify
\uglifytrue
\input{helpers}

\title{Satisfying Real-world Goals with Dataset Constraints}

\author{
  Gabriel Goh \\
  Dept. of Mathematics \\
  UC Davis \\
  Davis, CA 95616 \\
  \texttt{ggoh@math.ucdavis.edu}
  \and
  Andrew Cotter, Maya Gupta \\
  Google Inc. \\
  1600 Amphitheatre Parkway \\
  Mountain View, CA 94043 \\
  \texttt{\{acotter,mayagupta\}@google.com}
  \and
  Michael Friedlander \\
  Dept. of Computer Science \\
  University of British Columbia \\
  Vancouver, B.C. V6T 1Z4 \\
  \texttt{mpf@cs.ubc.ca}
}

\predate{}
\date{}
\postdate{}

\begin{document}

\label{document:begin}

\blfootnote{This is a slightly expanded version of that presented at the 30th Annual Conference on Neural Information Processing Systems (NIPS 2016).}

\maketitle

\begin{abstract}
\input{abstract}
\end{abstract}

\input{sec-goals}
\input{sec-problem}
\input{sec-related}
\input{sec-experiments}

\newpage
\clearpage

\bibliography{main}
\bibliographystyle{abbrvnat}

\label{document:middle}

\newpage
\clearpage

\appendix

\label{document:appendix}

\input{figures/tab-notation}

\input{app-randomized-classification}
\input{app-ratio-metrics}
\input{app-generalization}
\input{app-summary-examples}
\input{app-cutting-plane}
\input{app-svm}
\input{app-overall}

\label{document:end}

\end{document}

%% file: helpers.tex
\usepackage{amsmath}
\usepackage{amssymb}
\usepackage{amsthm}
\usepackage{appendix}
\usepackage{bbold}
\usepackage{color}
\usepackage{environ}
\usepackage{etoolbox}
\usepackage{float}
\usepackage{fullpage}
\usepackage{graphicx}
\usepackage{mathtools}
\usepackage{mdwlist}  % single-spaced itemize* and enumerate*
\usepackage{multirow}
\usepackage{natbib}
\usepackage{tabularx}
\usepackage{times}
\usepackage{titling}
\usepackage{url}
\usepackage{verbatim}
\usepackage{wrapfig}
\usepackage{xspace}

\usepackage{algorithm}
\usepackage{algorithmic}

\hypersetup{colorlinks = false, pdfborder = {0 0 0}}

\newcommand{\company}{Google}
\makeatletter
\def\blfootnote{\gdef\@thefnmark{}\@footnotetext}
\makeatother

\renewcommand{\eqref}[1]{Equation~\ref{eq:#1}}
\newcommand{\probref}[1]{Problem~\ref{pr:#1}}
\newcommand{\defref}[1]{Definition~\ref{def:#1}}
\newcommand{\secref}[1]{Section~\ref{sec:#1}}
\newcommand{\appref}[1]{Appendix~\ref{app:#1}}
\newcommand{\figref}[1]{Figure~\ref{fig:#1}}
\newcommand{\tabref}[1]{Table~\ref{tab:#1}}
\newcommand{\algref}[1]{Algorithm~\ref{alg:#1}}
\newcommand{\thmref}[1]{Theorem~\ref{thm:#1}}
\newcommand{\lemref}[1]{Lemma~\ref{lem:#1}}
\newcommand{\corref}[1]{Corollary~\ref{cor:#1}}
\newcommand{\defrefs}[2]{Definitions~\ref{def:#1} and~\ref{def:#2}}

\newcommand{\apprefs}[2]{Appendices~\ref{app:#1} and~\ref{app:#2}}

\newcommand{\algrefs}[2]{Algorithms~\ref{alg:#1} and~\ref{alg:#2}}

\newcommand{\lemrefs}[2]{Lemmas~\ref{lem:#1} and~\ref{lem:#2}}

\newcommand{\lemrefsss}[4]{Lemmas~\ref{lem:#1}, \ref{lem:#2}, \ref{lem:#3} and \ref{lem:#4}}

\ifdraft
  \newcommand{\TODO}[1]{\textcolor{red}{[TODO: #1]}\xspace}
  \newcommand{\NOTE}[1]{\textcolor{blue}{[NOTE: #1]}\xspace}
\else
  \newcommand{\TODO}[1]{\xspace}
  \newcommand{\NOTE}[1]{\xspace}
\fi

\ifuglify
  \newcommand{\eqcomma}{,}
  \newcommand{\eqperiod}{.}
\else
  \newcommand{\eqcomma}{}
  \newcommand{\eqperiod}{}
\fi

\newtheorem{problem}{Problem}
\newtheorem{theorem}{Theorem}
\newtheorem{lemma}{Lemma}
\newtheorem{corollary}{Corollary}
\newtheorem{defn}{Definition}

\NewEnviron{thm}[1]{
  \begin{theorem}\label{thm:#1}\BODY\end{theorem}
}
\NewEnviron{lem}[1]{
  \begin{lemma}\label{lem:#1}\BODY\end{lemma}
}
\NewEnviron{cor}[1]{
  \begin{corollary}\label{cor:#1}\BODY\end{corollary}
}

\renewcommand{\paragraph}[1]{\textbf{#1}\;\xspace}
\newcommand{\plotwidth}{0.47\textwidth}

\newcommand{\iid}{\emph{i.i.d.}\xspace}
\newcommand{\wrt}{w.r.t.\xspace}

\newcommand{\ie}{i.e.\xspace}
\newcommand{\eg}{e.g.\xspace}
\newcommand{\egcite}[1]{\citep[\eg][]{#1}}

\newcommand{\expectation}{\mathbb{E}}

\newcommand{\indicator}{\mathbf{1}}

\newcommand{\rademacher}{\mathcal{R}}

\newcommand{\R}{\mathbb{R}}

\newcommand{\maximize}[1][]{\underset{#1}{\mathrm{maximize\;}}}
\newcommand{\minimize}[1][]{\underset{#1}{\mathrm{minimize\;}}}
\newcommand{\subjectto}[1][]{\underset{#1}{\mathrm{s.t.\;}}}

\newcommand{\inner}[2]{\left\langle {#1}, {#2} \right\rangle}
\newcommand{\norm}[1]{\left\lVert {#1} \right\rVert}

\newcommand{\abs}[1]{\left\lvert {#1} \right\rvert}

\newcommand{\measureof}[2]{#1\left(#2\right)}

\DeclareMathOperator{\argmax}{argmax}
\DeclareMathOperator{\argmin}{argmin}

\newcommand{\codecomment}[1]{\;\;\;\;\;\;\;\;//~\textit{#1}}
\newcounter{lineno}
\newenvironment{pseudocode}{\setcounter{lineno}{0}\begin{tabbing}\textbf{mm}\=mm\=mm\=mm\=mm\=\kill}{\end{tabbing}}
\newcommand{\codename}{\>}

\newcommand{\codeline}{\>\stepcounter{lineno}\textbf{\arabic{lineno}}\'\>}

\newcommand{\lebesguemeasure}{\mu}
\newcommand{\dualvariables}{\xi}

\newcommand{\superlevelset}{S_{\ell}}
\newcommand{\superlevelhypograph}{S_h}

\newcommand{\svmproblem}{\digamma}

\newcommand{\multipliers}{v}
\newcommand{\multiplierbound}{V}
\newcommand{\multiplierspace}{\mathcal{V}}

\newcommand{\ramp}{\sigma}
\newcommand{\positiverampbound}{\check{\sigma}_p}
\newcommand{\negativerampbound}{\check{\sigma}_n}

\newcommand{\positiverate}{s_p}
\newcommand{\negativerate}{s_n}
\newcommand{\expectedpositiverate}{\bar{s}_p}
\newcommand{\expectednegativerate}{\bar{s}_n}
\newcommand{\ramppositiverate}{r_p}
\newcommand{\rampnegativerate}{r_n}
\newcommand{\ramppositiveratebound}{\check{r}_p}
\newcommand{\rampnegativeratebound}{\check{r}_n}

\newcommand{\positivelosscoefficient}[1]{\alpha_{#1}^{(0)}}
\newcommand{\negativelosscoefficient}[1]{\beta_{#1}^{(0)}}
\newcommand{\positiveconstraintcoefficient}[2]{\alpha_{#2}^{(#1)}}
\newcommand{\negativeconstraintcoefficient}[2]{\beta_{#2}^{(#1)}}
\newcommand{\constraintbound}[1]{\gamma^{(#1)}}

\newcommand{\positivesvmlosscoefficient}[1]{\check{\alpha}_{#1}^{(0)}}
\newcommand{\negativesvmlosscoefficient}[1]{\check{\beta}_{#1}^{(0)}}
\newcommand{\positivesvmconstraintcoefficient}[2]{\check{\alpha}_{#2}^{(#1)}}
\newcommand{\negativesvmconstraintcoefficient}[2]{\check{\beta}_{#2}^{(#1)}}
\newcommand{\svmconstraintbound}[1]{\check{\gamma}^{(#1)}}

\newcommand{\positivesvmcoefficient}[1]{\check{\alpha}_{#1}}
\newcommand{\negativesvmcoefficient}[1]{\check{\beta}_{#1}}

\newcommand{\ntruepositives}{\mbox{\#TP}}
\newcommand{\nfalsepositives}{\mbox{\#FP}}
\newcommand{\ntruenegatives}{\mbox{\#TN}}
\newcommand{\nfalsenegatives}{\mbox{\#FN}}
\newcommand{\nerrors}{\mbox{\#Errors}}
\newcommand{\nwins}{\mbox{\#Wins}}
\newcommand{\nlosses}{\mbox{\#Losses}}
\newcommand{\nchanges}{\mbox{\#Changes}}

%% file: abstract.tex
The goal of minimizing misclassification error on a training set is often just
one of several real-world goals that might be defined on different datasets.
For example, one may require a classifier to also make positive predictions at
some specified rate for some subpopulation (fairness), or to achieve a
specified empirical recall.
Other real-world goals include reducing churn with respect to a previously
deployed model, or stabilizing online training. In this paper we propose
handling multiple goals on multiple datasets by training with dataset
constraints, using the ramp penalty to accurately quantify costs, and present
an efficient algorithm to approximately optimize the resulting non-convex
constrained optimization problem.
Experiments on both benchmark and real-world industry datasets demonstrate the
effectiveness of our approach.

%% file: sec-goals.tex
\section{Real-world goals}\label{sec:goals}

We consider a broad set of design goals important for making classifiers work
well in real-world applications, and discuss how metrics quantifying many of
these goals can be represented in a particular optimization framework. The key
theme is that these metrics, which range from the standard precision and
recall, to less well-known examples such as coverage and
fairness~\citep{Mann:2007,Zafar:2015,Hardt:2016}, and including some new
proposals, can be expressed in terms of the positive and negative
classification rates on multiple datasets.

\paragraph{Coverage:}
One may wish to control how often a classifier predicts the positive (or
negative) class. For example, one may want to ensure that only $10\%$ of
customers are selected to receive a printed catalog due to budget constraints,
or perhaps to compensate for a biased training set. In practice, constraining
the ``coverage rate'' (the expected proportion of positive predictions) is
often easier than measuring \eg accuracy or precision because coverage can be
computed on unlabeled data---labeling data can be expensive, but acquiring a
large number of unlabeled examples is often very easy.

Coverage was also considered by \citet{Mann:2007}, who proposed what they call
``label regularization'', in which one adds a regularizer penalizing the
relative entropy between the mean score for each class and the desired
distribution, with an additional correction to avoid degeneracies.

\paragraph{Churn:}
Work does not stop once a machine learning model has been adopted. There will
be new training data, improved features, and potentially new model structures.
Hence, in practice, one will deploy a \emph{series} of models, each improving
slightly upon the last.
In this setting, determining whether each candidate should be deployed is
surprisingly challenging: if we evaluate on the \emph{same} held-out testing
set every time a new candidate is proposed, and deploy it if it outperforms its
predecessor, then every compare-and-deploy decision will increase the
statistical dependence between the deployed model and the testing dataset,
causing the model sequence to fit the originally-independent testing data.
This problem is magnified if, as is typical, the candidate models tend to
disagree only on a relatively small number of examples near the true decision
boundary.
%
%For example, with a fixed test set of $10\,000$ random examples, only $100$
%may be near the decision boundary, so the risk of the model sequence fitting
%these $100$ examples is heightened.

A simple and safe solution is to draw a \emph{fresh} testing sample every time
one wishes to compare two models in the sequence, only considering examples on
which the two models disagree. Because labeling data is expensive, one would
like these freshly sampled testing datasets to be as small as possible.
It is here that the problem of ``churn'' arises. Imagine that model A, our
deployed model, is $70\%$ accurate, and that model B, our candidate, is $75\%$
accurate.
In the best case, only $5\%$ of test samples would be labeled differently, and
all differences would be ``wins'' for classifier B. Then only a dozen or so
examples would need to be labeled in order to establish that B is the
statistically significantly better classifier with $95\%$ confidence.
In the worst case, model A would be correct and model B incorrect $25\%$ of
the time, model B correct and model A incorrect $30\%$ of the time, and both
models correct the remaining $45\%$ of the time. Then $55\%$ of testing
examples will be labeled differently, and closer to $1000$ examples would need
to be labeled to determine that model B is better.

We define the ``churn rate'' as the expected proportion of examples on which
the prediction of the model being considered (model B above) differs from that
of the currently-deployed model (model A). During training, we propose
constraining the empirical churn rate with respect to a given deployed model on
a large unlabeled dataset (see also \citet{Fard:2016} for an alternative
approach).

\paragraph{Stability:}
A special case of minimizing churn is to ensure stability of an online
classifier as it evolves, by constraining it to not deviate too far from a
trusted classifier on a large held-out unlabeled dataset.
%
%Classifier stability may also be important for follow-on logic or end-users.

\paragraph{Fairness:}
A practitioner may be required to guarantee \emph{fairness} of a learned
classifier, in the sense that it makes positive predictions on different
subgroups at certain rates. For example, one might require that housing loans
be given equally to people of different genders. \citet{Hardt:2016} identify
three types of fairness: (i) demographic parity, in which positive predictions
are made at the same rate on each subgroup, (ii) equal opportunity, in which
only the true positive rates must match, and (iii) equalized odds, in which
both the true positive rates and false positive rates must match. Fairness can
also be specified by a proportion, such as the $80\%$ rule in US law that
certain decisions must be in favor of group B individuals at least $80\%$ as
often as group A
individuals~\egcite{Biddle:2005,Vuolo:2013,Zafar:2015,Hardt:2016}.

\citet{Zafar:2015} propose learning fair classifiers by imposing linear
constraints on the covariance between the predicted labels and the values of
certain features, while \citet{Hardt:2016} propose first learning an ``unfair''
classifier, and then choosing population-dependent thresholds to satisfy the
desired fairness criterion. In our framework, rate constraints such as those
mentioned above can be imposed directly, at training time.

%\citet{Zafar:2015} propose constraining classifiers to obey notions of
%fairness, for example, ensuring that housing loans are given equally to people
%of different genders. As they note, fairness is sometimes specified by a
%proportion, such as the $80\%$ rule in US law that certain decisions must be in
%favor of group B individuals $80\%$ as often as in favor of group A individuals
%\egcite{Vuolo:2013}.

\paragraph{Recall and Precision:}
Requirements of real-world classifiers are often expressed in terms of
precision and recall, especially when examples are highly imbalanced between
positives and negatives.
In our framework, we can handle this problem via Neyman-Pearson
classification~\egcite{Scott:2005, Davenport:2010}, in which one seeks to
minimize the false negative rate subject to a constraint on the false positive
rate. Indeed, our ramp-loss formulation is equivalent to that of
\citet{Bottou:2011} in this setting.

\paragraph{Egregious Examples:}
For certain classification applications, examples may be discovered that are
particularly embarrassing if classified incorrectly. One standard approach to
handling such examples is to increase their weights during training, but this
is difficult to get right: too large a weight may distort the classifier too
much in the surrounding feature space, whereas too small a weight may not fix
the problem. Worse, over time the dataset will often be augmented with new
training examples and new features, causing the ideal weights to drift. We
propose instead simply adding a constraint ensuring that some proportion of a
set of such egregious examples is correctly classified.  Such constraints
should be used with extreme care, since they can cause the problem to become
infeasible.
%
%Correct classification of a set of egregious examples also provides a quick
%summarization bit that aids in interpretability and understanding of the
%trained model.

%% file: sec-problem.tex
\section{Optimization problem}\label{sec:problem}

\input{figures/tab-datasets}
\input{figures/tab-metrics}
A key aspect of many of the goals of \secref{goals} is that they are defined on
different datasets. For example, we might seek to maximize the accuracy on a
set of labeled examples drawn in some biased manner, require that its recall be
at least $90\%$ on $50$ small datasets sampled in an unbiased manner from $50$
different countries, desire low churn relative to a deployed classifier on a
large unbiased unlabeled dataset, and require that $100$ given egregious
examples be classified correctly.

Another characteristic common to the metrics of \secref{goals} is that they can
be expressed in terms of the positive and negative classification rates on
various datasets. We consider only \emph{unlabeled} datasets, as described in
\tabref{datasets}---a dataset with binary labels, for example, would be handled
by partitioning it into the two unlabeled datasets $D^+$ and $D^-$ containing
the positive and negative examples, respectively. We wish to learn a linear
classification function $f(x) = \inner{w}{x} - b$ parameterized by a weight
vector $w\in\R^d$ and bias $b\in\R$, for which the positive and negative
classification rates are:
\begin{equation}
  \label{eq:rates} \positiverate\left(D; w, b\right) = \tfrac{1}{|D|}
  \textstyle{\sum}_{x \in D} \indicator\left( \inner{w}{x} - b \right),
  \;\;\;\;\;\;\;\;
  \negativerate\left(D; w, b\right) = \positiverate \left(D; -w, -b\right)
  \eqcomma
\end{equation}
where $\indicator$ is an indicator function that is $1$ if its argument is
positive, $0$ otherwise. In words, $\positiverate(D; w, b)$ and
$\negativerate(D; w, b)$ denote the proportion of positive or negative
predictions, respectively, that $f$ makes on $D$. \tabref{metrics} specifies
how the metrics of \secref{goals} can be expressed in terms of the
$\positiverate$s and $\negativerate$s.

We propose handling these goals by minimizing an $\ell^2$-regularized positive
linear combination of prediction rates on different datasets, subject to
upper-bound constraints on other positive linear combinations of such
prediction rates:
\begin{problem}\label{pr:original-objective}
  Starting point: discontinuous constrained problem
  \begin{align*}
    \minimize[ w\in\R^d, b\in\R ] & \quad
    \textstyle{ \sum_{i=1}^k}
    \left( \positivelosscoefficient{i} \positiverate ( D_i; w, b
    ) + \negativelosscoefficient{i} \negativerate ( D_i; w, b
    ) \right) + \frac{\lambda}{2} \norm{w}_2^2 \\
    \subjectto & \quad \textstyle{\sum_{i=1}^k} \left(
    \positiveconstraintcoefficient{j}{i}
    \positiverate ( D_i; w, b
    ) + \negativeconstraintcoefficient{j}{i} \negativerate ( D_i; w, b )
    \right) \leq \constraintbound{j} \quad j \in \{1, \dots, m\} \eqperiod
  \end{align*}
\end{problem}
Here, $\lambda$ is the parameter on the $\ell^2$ regularizer, there are $k$
unlabeled datasets $D_1, \dots, D_k$ and $m$ constraints. The metrics minimized
by the objective and bounded by the constraints are specified via the choices
of the nonnegative coefficients $\positivelosscoefficient{i}$,
$\negativelosscoefficient{i}$, $\positiveconstraintcoefficient{j}{i}$,
$\negativeconstraintcoefficient{j}{i}$ and upper bounds $\constraintbound{j}$
for the $i$th dataset and, where applicable, the $j$th constraint---a user
should base these choices on \tabref{metrics}.
Note that because $\positiverate + \negativerate = 1$, it is possible to
transform \emph{any} linear combination of rates into an equivalent positive
linear combination, plus a constant (see
\appref{ratio-metrics}\footnote{Appendices may be found in the supplementary
material} for an example).

We cannot optimize \probref{original-objective} directly because the rate
functions $\positiverate$ and $\negativerate$ are discontinuous.
%
% NP-hardness~\citep{Nguyen:2013}
%
We can, however, work around this difficulty by training a classifier that
makes \emph{randomized} predictions based on the ramp
function~\citep{Collobert:2006}:
\begin{equation}
  \label{eq:ramp-loss} \ramp(z) = \max \{ 0, \min\{ 1, \nicefrac{1}{2} + z \}
  \} \eqcomma
\end{equation}
where the randomized classifier parameterized by $w$ and $b$ will make a
positive prediction on $x$ with probability $\ramp\left(\inner{w}{x} -
b\right)$, and a negative prediction otherwise (see
\appref{randomized-classification} for more on this randomized classification
rule).
For this randomized classifier, the \emph{expected} positive and negative rates
will be:
\begin{equation}
  \label{eq:ramp-rates} \ramppositiverate\left(D;w,b\right) =
  \tfrac{1}{\abs{D}}\textstyle{\sum}_{x\in D}\ramp\left(\inner{w}{x} -
  b\right),
  \;\;\;\;\;\;\;\;
  \rampnegativerate\left(D;w,b\right) = \ramppositiverate\left(D;-w,-b\right)
  \eqperiod
\end{equation}
Using these expected rates yields a continuous (but non-convex) analogue of
\probref{original-objective}:
\begin{problem}\label{pr:ramp-objective}
  Ramp version of \probref{original-objective}
  \begin{align*}
    \minimize[ w\in\R^d, b\in\R ] & \quad \textstyle{ \sum_{i=1}^k} \left(
    \positivelosscoefficient{i} \ramppositiverate ( D_i; w, b) +
    \negativelosscoefficient{i} \rampnegativerate ( D_i; w, b) \right) +
    \frac{\lambda}{2} \norm{w}_2^2 \\
    \subjectto & \quad \textstyle{\sum_{i=1}^k} \left(
    \positiveconstraintcoefficient{j}{i} \ramppositiverate ( D_i; w, b) +
    \negativeconstraintcoefficient{j}{i} \rampnegativerate ( D_i; w, b )
    \right) \leq \constraintbound{j} \quad j \in \{1, \dots, m\} \eqperiod
  \end{align*}
\end{problem}
Efficient optimization of this problem is the ultimate goal of this section. In
\secref{problem:optimization}, we will propose a majorization-minimization
approach that sequentially minimizes convex upper bounds on
\probref{ramp-objective}, and, in \secref{problem:cutting-plane}, will discuss
how these convex upper bounds may themselves be efficiently optimized.

\subsection{Optimizing the ramp problem}\label{sec:problem:optimization}

\input{figures/alg-majorization-minimization}
\input{figures/fig-ramp-bounds}
To address the non-convexity of \probref{ramp-objective}, we will iteratively
optimize approximations, by, starting from an feasible initial candidate
solution, constructing a convex optimization problem upper-bounding
\probref{ramp-objective} that is \emph{tight} at the current candidate,
optimizing this convex problem to yield the next candidate, and repeating.

Our choice of a ramp for $\ramp$ makes finding such tight convex upper bounds
easy: both the hinge function $\max\left\{ 0,\nicefrac{1}{2}+z\right\}$ and
constant-$1$ function are upper bounds on $\ramp$, with the former being tight
for all $z\le\nicefrac{1}{2}$, and the latter for all $z\ge\nicefrac{1}{2}$
(see \figref{ramp-bounds}). We'll therefore define the following upper bounds
on $\ramp$ and $1-\ramp$, with the additional parameter $z'$ determining which
of the two bounds (hinge or constant) will be used, such that the bounds will
always be tight for $z=z'$:
\begin{equation}
  \label{eq:ramp-bounds} \positiverampbound\left(z; z'\right) = \begin{cases}
    \max\left\{ 0, \nicefrac{1}{2} + z \right\}& \;\;\;\;\mbox{if }z' \le
    \nicefrac{1}{2} \\
    1 & \;\;\;\;\mbox{otherwise}
  \end{cases},
  \;\;\;\;\;\;\;\;
  \negativerampbound(z;z') = \positiverampbound\left(-z; -z'\right) \eqperiod
\end{equation}
Based upon these we define the following upper bounds on the expected rates:
\begin{align}
  \label{eq:rate-bounds} \ramppositiveratebound\left( D; w, b; w', b' \right)=
  & \tfrac{1}{\abs{D}} \textstyle{\sum_{x\in D}} \,\, \positiverampbound\left(
  \inner{w}{x} - b; \inner{w'}{x} - b' \right) \\
  \notag \rampnegativeratebound\left( D; w, b; w', b' \right)= &
  \tfrac{1}{\abs{D}} \textstyle{\sum_{x\in D}} \,\, \negativerampbound\left(
  \inner{w}{x} - b; \inner{w'}{x} - b' \right) \eqcomma
\end{align}
which have the properties that both $\ramppositiveratebound$ and
$\rampnegativeratebound$ are convex in $w$ and $b$, are upper bounds on the
original ramp-based rates:
\begin{equation*}
  \ramppositiveratebound\left( D; w, b; w', b' \right) \ge
  \ramppositiverate\left( D; w, b \right)
  \;\;\;\;\mbox{ and }\;\;\;\;
  \rampnegativeratebound\left( D; w, b; w', b' \right)\ge
  \rampnegativerate\left( D; w, b \right) \eqcomma
\end{equation*}
and are tight at $w',b'$:
\begin{equation*}
  \ramppositiveratebound\left( D; w', b' ; w', b' \right)=
  \ramppositiverate\left( D; w', b' \right)
  \;\;\;\;\mbox{ and }\;\;\;\;
  \rampnegativeratebound\left( D; w', b'; w', b' \right)=
  \rampnegativerate\left( D; w', b' \right) \eqperiod
\end{equation*}
Substituting these bounds into \probref{ramp-objective} yields:
\begin{problem}\label{pr:convex-objective}
  Convex upper bound on \probref{ramp-objective}
  \begin{align*}
    \minimize[w\in\R^d, b\in\R] & \quad \textstyle{\sum_{i=1}^{k}} \left(
    \positivelosscoefficient{i} \ramppositiveratebound\left( D_{i}; w, b; w',
    b' \right) + \negativelosscoefficient{i} \rampnegativeratebound\left(
    D_{i}; w, b; w', b' \right) \right) + \frac{\lambda}{2} \norm{ w }_2^2 \\
    \subjectto & \quad \textstyle{\sum_{i=1}^{k}} \left(
    \positiveconstraintcoefficient{j}{i} \ramppositiveratebound\left( D_{i}; w,
    b; w', b' \right) + \negativeconstraintcoefficient{j}{i}
    \rampnegativeratebound\left( D_{i}; w, b; w', b' \right) \right) \le
    \constraintbound{j} \quad j \in \{1, \dots, m\} \eqperiod
  \end{align*}
\end{problem}
As desired, this problem upper bounds \probref{ramp-objective}, is tight at
$w', b'$, and is convex (because any positive linear combination of convex
functions is convex).

\algref{majorization-minimization} contains our proposed procedure for
approximately solving \probref{ramp-objective}.
%
%Starting from an initial feasible solution $(w^{(0)}, b_0)$, we repeatedly
%find a convex upper bound problem that is tight at the current candidate
%solution (line~2), and optimize it to yield the next candidate (line~3).
%
Given an initial feasible solution, it's straightforward to verify inductively,
using the fact that we construct tight convex upper bounds at every step, that
every convex subproblem will have a feasible solution, every $(w^{(t)}, b_t)$
pair will be feasible \wrt \probref{ramp-objective}, and every $(w^{(t + 1)},
b_{t + 1})$ will have an objective function value that is no larger that that
of $(w^{(t)}, b_t)$. In other words, no iteration can make negative progress.
%
%\TODO{argue that it will generally make positive progress, unless we have
%extremely (measure-zeroish) bad luck}
%
The non-convexity of \probref{ramp-objective}, however, will cause
\algref{majorization-minimization} to arrive at a suboptimal solution that
depends on the initial $(w^{(0)},b_0)$.

\subsection{Optimizing the convex subproblems}\label{sec:problem:cutting-plane}

\input{figures/alg-cutting-plane-multipliers}
The first step in optimizing \probref{convex-objective} is to add Lagrange
multipliers $\multipliers$ over the constraints, yielding the equivalent
unconstrained problem:
\begin{equation}
  \label{eq:dual-problem} \maximize[\multipliers\succeq0]\
  \svmproblem(\multipliers) = \min_{w,b} \Psi\left( w, b, \multipliers; w', b'
  \right) \eqcomma
\end{equation}
where the function:
\begin{align}
  \label{eq:psi-definition} \Psi\left(w,b,\multipliers;w',b'\right) =
  & \textstyle{\sum_{i=1}^{k}} \left( \left(\positivelosscoefficient{i} +
  \textstyle{\sum_{j=1}^m} \multipliers_j \positiveconstraintcoefficient{j}{i}
  \right) \ramppositiveratebound\left( D_{i}; w, b; w', b' \right) \right. \\
  \notag & \left. + \left(\negativelosscoefficient{i} +
  \textstyle{\sum_{j=1}^m} \multipliers_j \negativeconstraintcoefficient{j}{i}
  \right) \rampnegativeratebound\left( D_{i}; w, b; w', b' \right) \right)
  + \tfrac{\lambda}{2} \norm{w}_2^2
  - \textstyle{\sum_{j=1}^m} \multipliers_j \constraintbound{j}
\end{align}
is convex in $w$ and $b$, and concave in the multipliers $\multipliers$. For
the purposes of this section, $w'$ and $b'$, which were found in the previous
iteration of \algref{majorization-minimization}, are fixed constants.

Because this is a convex-concave saddle point problem, there are a large number
of optimization techniques that could be successfully applied. For example, in
settings similar to our own, \citet{Eban:2016} simply perform SGD jointly over
all parameters (including $\multipliers$), while \citet{Bottou:2011} use the
Uzawa algorithm, which would alternate between (i) optimizing exactly over $w$
and $b$, and (ii) taking gradient steps on $\multipliers$.

We instead propose an approach for which, in our setting, it is particularly
easy to create an efficient implementation.
The key insight is that evaluating $\svmproblem(\multipliers)$ is, thanks to
our use of hinge and constant upper-bounds on our ramp $\ramp$, equivalent to
optimization of a support vector machine (SVM) with per-example weights---see
\appref{SVM} for details. This observation enables us to solve the saddle
system in an inside-out manner. On the ``inside'', we optimize over $(w, b)$
for fixed $\multipliers$ using an off-the-shelf SVM solver~\egcite{LibSVM}. On
the ``outside'', the resulting $(w,b)$-optimizer is used as a component in a
cutting-plane optimization over $\multipliers$. Notice that this outer
optimization is very low-dimensional, since $\multipliers \in \R^m$, where $m$
is the number of constraints.

\algref{cutting-plane-multipliers} contains a skeleton of the cutting-plane
algorithm that we use for this outer optimization over $\multipliers$.
Because this algorithm is intended to be used as an outer loop in a nested
optimization routine, it does not expect that $\svmproblem(\multipliers)$ can be
evaluated or differentiated exactly. Rather, it's based upon the idea of
possibly making ``shallow'' cuts~\citep{Bland:1981} by choosing a desired
accuracy $\epsilon_t$ at each iteration, and expecting the SVMOptimizer to
return a solution with suboptimality $\epsilon_t$. More precisely, the
SVMOptimizer function approximately evaluates $\svmproblem(\multipliers^{(t)})$
for a given fixed $\multipliers^{(t)}$ by constructing the corresponding SVM
problem and finding a $(w^{(t)}, b_t)$ for which the primal and dual
objective function values differ by at most $\epsilon_t$.
%
%(observe that by strong duality, the primal objective is never smaller than
%the optimum, and the dual never larger, with the two being equal only at the
%optimum).

After finding $(w^{(t)},b_t)$, the SVMOptimizer then evaluates the dual
objective function value of the SVM to determine $l_t$. The primal objective
function value $u_t$ and its gradient $g^{(t)}$ \wrt $\multipliers$ (calculated
on line 10 of \algref{cutting-plane-multipliers}) define the cut $u_t +
\inner{g^{(t)}}{\multipliers - \multipliers^{(t)}}$.
Notice that since $\Psi(w^{(t)}, b_t, \multipliers; w', b')$ is a linear
function of $\multipliers$, it is equal to this cut function, which therefore
upper-bounds $\min_{w,b} \Psi(w, b, \multipliers; w', b')$.

One advantage of this cutting-plane formulation is that typical CutChooser
implementations will choose $\epsilon_t$ to be large in the early iterations,
and will only shrink it to be $\epsilon$ or smaller once we're close to
convergence. We leave the details of the analysis to
\apprefs{cutting-plane}{SVM}---a summary can be found in \appref{overall}.

%\TODO{discuss $\multiplierspace$, $l_0$ and $u_0$}

%% file: figures/tab-datasets.tex
\begin{table*}[t]

\caption{Dataset notation.}

\label{tab:datasets}

\begin{center}

\begin{tabularx}{\textwidth}{lX}
  \hline
  \textbf{Notation} & \textbf{Dataset}\\
  \hline
  $D$ & Any dataset \\
  $D^+, D^-$ & Sets of examples labeled positive/negative, respectively \\
  $D^{++}, D^{+-}, D^{-+}, D^{--}$ & Sets of examples with ground-truth
  positive/negative labels, and for which a baseline classifier makes
  positive/negative predictions \\
  $D^A$, $D^B$ & Sets of examples belonging to subpopulation A and B,
  respectively \\
  %
  %$D^{R_p}$ and $D^{R_n}$ & Set of examples that satisfy a positive (or
  %negative) decision rule $R$ \\
  %
  \hline
\end{tabularx}

\end{center}

\end{table*}

%% file: figures/tab-metrics.tex
\begin{table*}[t]

\caption{
  The quantities discussed in \secref{goals}, expressed in the notation used in
  \probref{original-objective}, with the dependence on $w$ and $b$ dropped for
  notational simplicity, and using the dataset notation of \tabref{datasets}.
}

\label{tab:metrics}

\begin{center}

\begin{tabularx}{\textwidth}{lX}
  \hline
  \textbf{Metric} & \textbf{Expression} \\
  \hline
  Coverage rate & $\positiverate\left(D\right)$ \\
  $\ntruepositives$, $\ntruenegatives$, $\nfalsepositives$, $\nfalsenegatives$
  & $\abs{D^+} \positiverate\left(D^+\right)$, $\abs{D^-}
  \negativerate\left(D^-\right)$, $\abs{D^-} \positiverate\left(D^-\right)$,
  $\abs{D^+} \negativerate\left(D^+\right)$ \\
  $\nerrors$ & $\nfalsepositives + \nfalsenegatives$ \\
  Error rate & $\nerrors / \left(\abs{D^+} + \abs{D^-}\right)$ \\
  Recall & $\ntruepositives / \left(\ntruepositives + \nfalsenegatives\right) =
  \ntruepositives / \abs{D^+}$ \\
  $\nchanges$ & $\abs{D^{+-}} \positiverate\left(D^{+-}\right) + \abs{D^{-+}}
  \negativerate\left(D^{-+}\right)$  + $\abs{D^{+-}}
  \positiverate\left(D^{+-}\right) + \abs{D^{-+}}
  \negativerate\left(D^{-+}\right)$ \\
  Churn rate & $\nchanges / \left(\abs{D^{++}} + \abs{D^{+-}} + \abs{D^{-+}} +
  \abs{D^{--}}\right)$ \\
  %
  %Unfairness & $\positiverate\left(D^A\right) - \positiverate\left(D^B\right)$
  %\\
  %
  %Fairness & $\positiverate\left(D^A\right) - \positiverate\left(D^B\right)$ \\
  %
  Fairness constraint & $\positiverate\left(D^A\right) \geq \kappa
  \positiverate\left(D^B\right)$, where $\kappa > 0$ \\
  Equal opportunity constraint & $\positiverate\left(D^A \cap D^+\right) \geq
  \kappa \positiverate\left(D^B \cap D^+\right)$, where $\kappa > 0$ \\
  Egregious example constraint & $\positiverate\left(D^+\right) \geq \kappa$
  and/or $\negativerate\left(D^-\right) \leq \kappa$ for a dataset $D$ of
  egregious examples, where $\kappa \in [0,1]$ \\
  %
  %Decision Rules & $\positiverate\left(D^{R_p}\right) \geq 1$ or
  %$\negativerate\left(D^{R_n}\right) \leq -1$, for a positive or negative
  %rule\\
  %
  \hline
\end{tabularx}

\end{center}

\end{table*}

%% file: figures/alg-majorization-minimization.tex
\begin{algorithm*}[t]

\begin{pseudocode}
\codename $\mbox{MajorizationMinimization}\left( w^{(0)}, b_0, T \right)$ \\
\codeline For $t \in \{1, 2, \dots, T\}$ \\
\codeline \>Construct an instance of \probref{convex-objective} with $w' = w^{(t-1)}$ and $b' = b_{t-1}$ \\
\codeline \>Optimize this convex optimization problem to yield $w^{(t)}$ and $b_t$ \\
\codeline Return $w^{(t)}$, $b_t$
\end{pseudocode}

\caption{
  Proposed majorization-minimization procedure for (approximately) optimizing
  \probref{ramp-objective}. Starting from an initial feasible solution
  $w^{(0)}$, $b_0$, we repeatedly find a convex upper bound problem that is
  tight at the current candidate solution, and optimize it to yield the next
  candidate.
  See \secref{problem:optimization} for details, and
  \secref{problem:cutting-plane} for how one can perform the inner
  optimizations on line 3.
}

\label{alg:majorization-minimization}

\end{algorithm*}

%% file: figures/fig-ramp-bounds.tex
\begin{figure*}[t]

\begin{center}

\begin{tabular}{cc}

\includegraphics[width=\plotwidth]{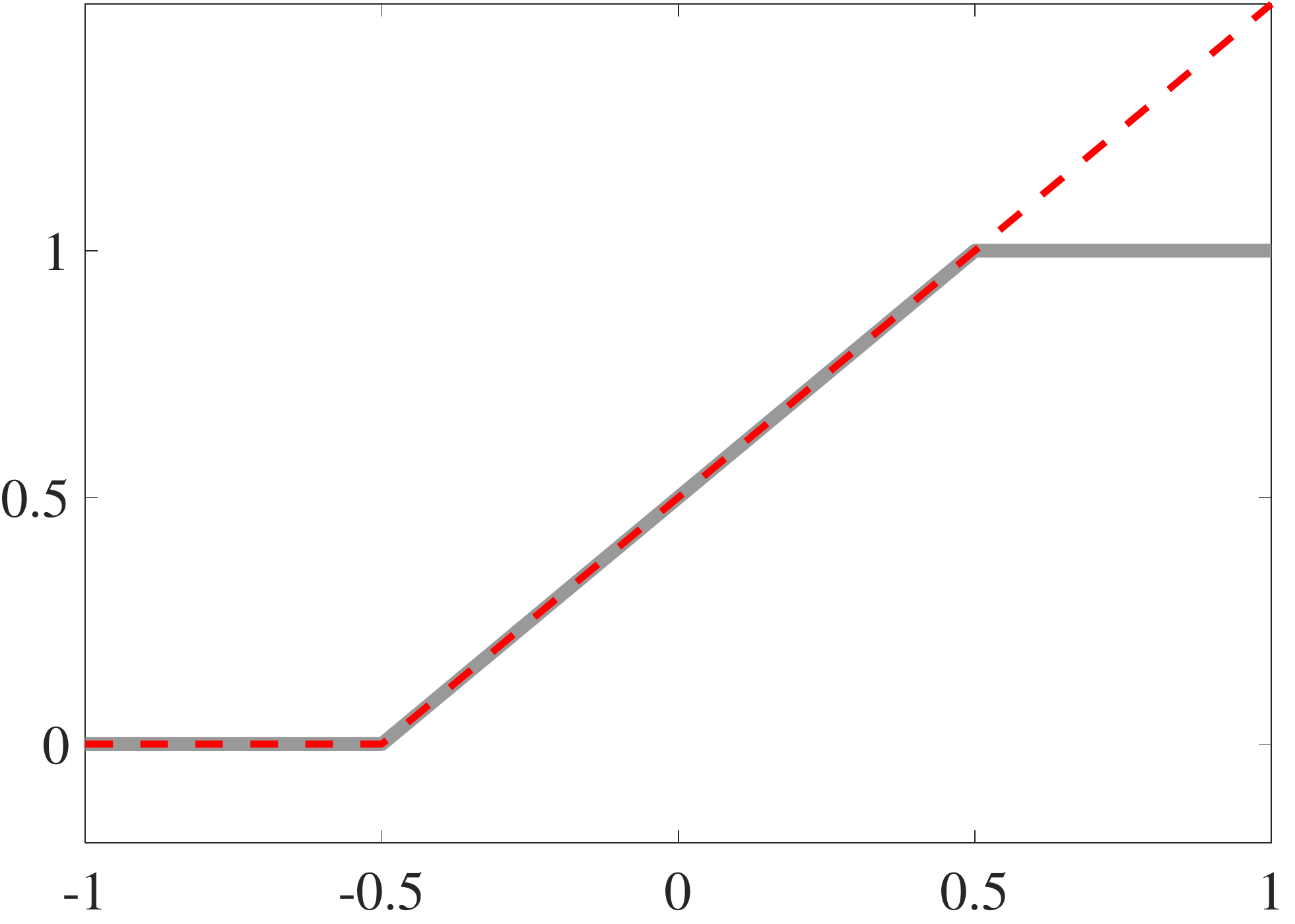} &

\includegraphics[width=\plotwidth]{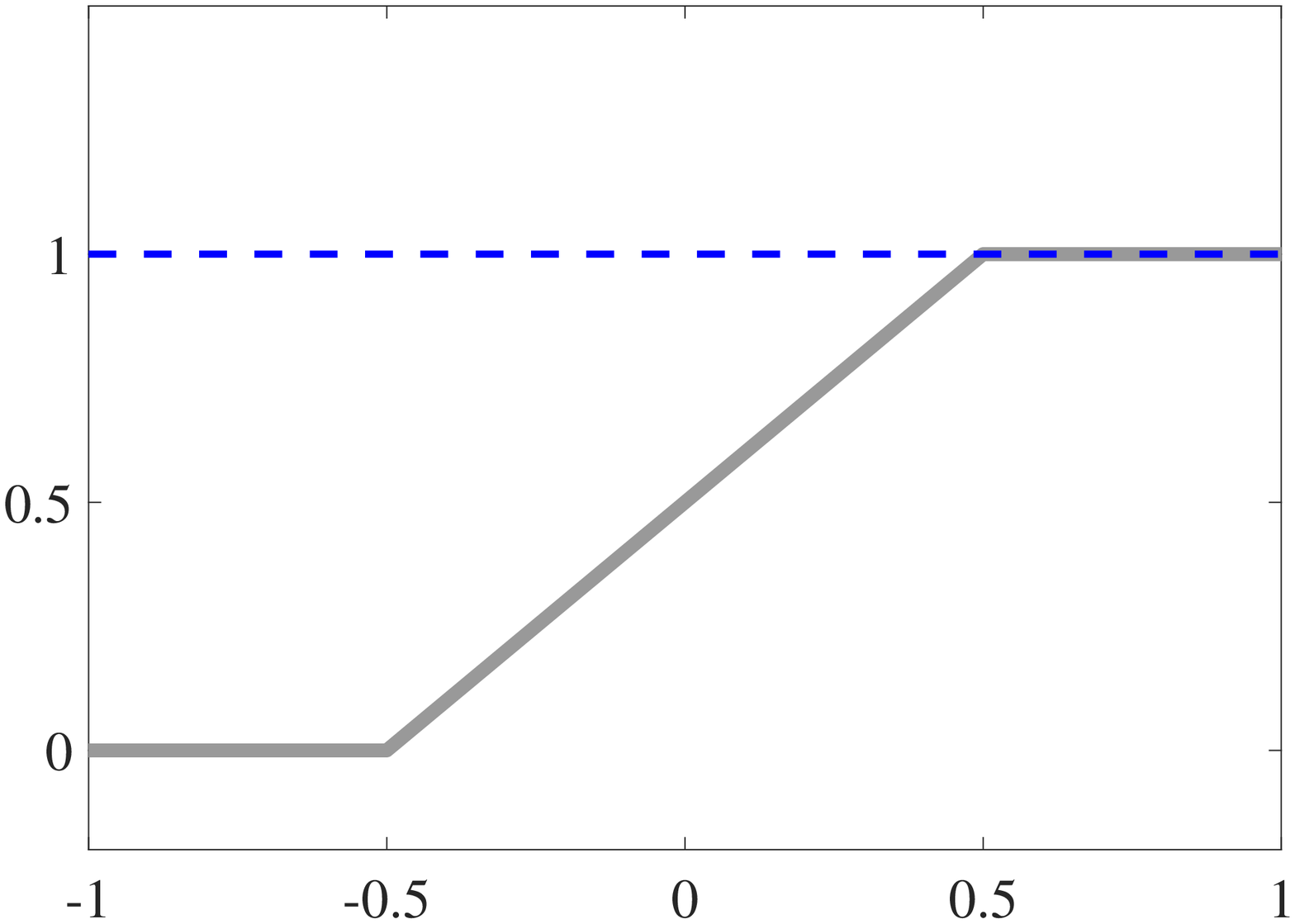}

\end{tabular}

\end{center}

\caption{
  Convex upper bounds on the ramp function $\sigma(z)=\max\left\{ 0,\min\left\{
  1,\nicefrac{1}{2}+z\right\} \right\}$. Notice that the hinge bound (left) is
  tight for all $z\le\nicefrac{1}{2}$, and the constant bound (right) is tight
  for all $z\ge\nicefrac{1}{2}$.
}

\label{fig:ramp-bounds}

\end{figure*}

%% file: figures/alg-cutting-plane-multipliers.tex
\begin{algorithm*}[t]

\begin{pseudocode}
\codename $\mbox{CuttingPlane}\left( l_0, u_0, \multiplierspace, \epsilon \right)$ \\
\codeline Initialize $g^{(0)} \in \R^m$ to the all-zero vector \\
\codeline For $t \in \{1, 2, \dots\}$ \\
\codeline \>Let $h_t\left(\multipliers\right) = \min_{s\in\{ 0, 1, \dots, t - 1 \}} \left( u_s + \inner{g^{(s)}}{\multipliers - \multipliers^{(s)}} \right)$ \\
\codeline \>Let $L_t = \max_{s \in \{ 0, 1, \dots, t - 1 \} } l_s$ and $U_t = \max_{\multipliers \in \multiplierspace} h_t\left( \multipliers \right)$ \\
\codeline \>If $U_t - L_t \le \epsilon$ then \\
\codeline \>\>Let $s \in \{ 1, \dots, t - 1\}$ be an index maximizing $l_s$ \\
\codeline \>\>Return $w^{(s)}$, $b_s$, $\multipliers^{(s)}$ \\
%\codeline \>\emph{Somehow} choose a $\multipliers^{(t)} \in \multiplierspace$ for which $h_t\left( \multipliers^{(t)} \right) \ge L_t$, and an $\epsilon_t > 0$ \\
\codeline \>Let $\multipliers^{(t)}, \epsilon_t = \mbox{CutChooser}\left(h_t, L_t\right)$ \\
\codeline \>Let $w^{(t)}, b_t, l_t = \mbox{SVMOptimizer}\left( \multipliers^{(t)}, h_t\left(\multipliers^{(t)}\right), \epsilon_t \right)$ \\
\codeline \>Let $u_t = \Psi( w^{(t)}, b_t, \multipliers^{(t)}; w', b')$ and $g^{(t)} = \nabla_{\multipliers} \Psi( w^{(t)}, b_t, \multipliers^{(t)}; w', b')$
\end{pseudocode}

\caption{
  Skeleton of a cutting-plane algorithm that optimizes \eqref{dual-problem} to
  within $\epsilon$ for $v\in\multiplierspace$, where $\multiplierspace
  \subseteq \R^m$ is compact and convex.
  Here, $l_0,u_0 \in \R$ are finite with $l_0 \le \max_{\multipliers \in
  \multiplierspace} \svmproblem(\multipliers) \le u_0$.
  There are several options for the CutChooser function on line 8---please see
  \appref{cutting-plane} for details.
  The SVMOptimizer function returns $w^{(t)}$ and $b_t$ approximately
  minimizing $\Psi(w, b, \multipliers^{(t)}; w', b')$, and a lower bound $l_t
  \le \svmproblem(\multipliers)$ for which $u_t - l_t \le \epsilon_t$ for $u_t$
  as defined on line 10.
}

\label{alg:cutting-plane-multipliers}

\end{algorithm*}

%% file: sec-related.tex
\section{Related work}\label{sec:related}

The problem of finding optimal trade-offs in the presence of multiple
objectives has been studied generically in the field of multi-objective
optimization~\citep{Miettinen:2012}. Two common approaches are (i) linear
scalarization~\citep[Section 3.1]{Miettinen:2012}, and (ii) the method of
$\epsilon$-constraints~\citep[Section 3.2]{Miettinen:2012}. Linear
scalarization reduces to the common heuristic of reweighting groups of
examples. The method of $\epsilon$-constraints puts hard bounds on the
magnitudes of secondary objectives, like our dataset constraints.
%
%This latter approach, while computationally less convenient, lends itself to
%interpretable parameters which, in our setting, have concrete, real world
%meanings.
%
Notice that, in our formulation, the Lagrange multipliers $\multipliers$ play
the role of the weights in the linear scalarization approach, with the
difference being that, rather than being provided directly by the user, they
are dynamically chosen to satisfy constraints. The user controls the problem
through these constraint choices, which have concrete real-world meanings.

%Dataset constraints are also related to chance constrained optimization.
While the hinge loss is one of the most commonly-used convex upper bounds on
the $0/1$ loss~\citep{Rockafellar:2000}, we use the ramp loss, trading off
convexity for tightness. For our purposes, the main disadvantage of the hinge
loss is that it is unbounded, and therefore cannot distinguish a single very
bad example from say, 10 slightly bad ones, making it ill-suited for
constraints on rates. In contrast, for the ramp loss the contribution of any
single datum is bounded, no matter how far it is from the decision boundary.

The ramp loss has also been investigated in \citet{Collobert:2006} (without
constraints). \citet{Bottou:2011} use the ramp loss both in the objective and
constraints, but their algorithm only tackles the Neyman-Pearson problem. They
compared their classifier to that of \citet{Davenport:2010}, which differs in
that it uses a hinge relaxation instead of the ramp loss, and found with the
ramp loss they achieved similar or slightly better results with up to
$10\times$ less computation (our approach does not enjoy this computational
speedup).

\citet{Narasimhan:2015} considered optimizing the F-measure and other
quantities that can be written as concave functions of the TP and TN rates.
Their proposed stochastic dual solver adaptively linearizes concave functions
of the rate functions (\eqref{rates}). \citet{Joachims:2005b} indirectly
optimizes upper-bounds on functions of $\positiverate(D^+)$,
$\positiverate(D^-)$, $\negativerate(D^+)$, $\negativerate(D^-)$ using a hinge
loss approximation.

Finally, for some simple problems (particularly when there is only one
constraint), the goals in \secref{goals} can be coarsely handled by simple
bias-shifting, \ie first training an unconstrained classifier, and then
attempting to adjust the decision threshold to satisfy the constraints as a
second step.

%For a survey of more prior work on optimizing measures other than accuracy with
%machine learning, see also \citet{Joachims:2005b}.

%% file: sec-experiments.tex
\section{Experiments}\label{sec:experiments}

We evaluate the performance of the proposed approach in two experiments, the
first using a benchmark dataset for fairness, and the second on a real-world
problem with churn and recall constraints.

\subsection{Fairness}\label{sec:experiments:fairness}

\input{figures/fig-fairness}
We compare training for fairness on the Adult dataset~\footnote{``a9a'' from
\url{https://www.csie.ntu.edu.tw/~cjlin/libsvmtools/datasets/binary.html}}, the
same dataset used by \citet{Zafar:2015}. The $32\,561$ training and $16\,281$
testing examples, derived from the 1994 Census, are $123$-dimensional and
sparse. Each feature contains categorical attributes such as race, gender,
education levels and relationship status. A positive class label means that
individual's income exceeds 50k. Let $D^M$ and $D^F$ denote the sets of male
and female examples. The number of positive labels in $D^M$ is roughly six
times that of $D^F$. The goal is to train a classifier that respects the
fairness constraint $\positiverate\left( D^M \right) \leq \positiverate\left(
D^F \right) / \kappa$ for a parameter $\kappa \in (0,1]$ (where $\kappa=0.8$
corresponds to the $80\%$ rule mentioned in \secref{goals}).

Our publicly-available \texttt{Julia}
implementation\footnote{\url{https://github.com/gabgoh/svmc.jl}} for these
experiments uses \texttt{LIBLINEAR}~\citep{LibLinear} with the default
parameters (most notably $\lambda = 1/n \approx 3\times{10}^{-5}$) to implement
the SVMOptimizer function, and does not include an unregularized bias $b$. The
outer optimization over $\multipliers$ does not use the $m$-dimensional cutting
plane algorithm of \algref{cutting-plane-multipliers}, instead using a simpler
one-dimensional variant (observe that these experiments involve only one
constraint). The majorization-minimization procedure starts from the
all-zeros vector ($w^{(0)}$ in \algref{majorization-minimization}).

We compare to the method of \citet{Zafar:2015}, which proposed handling
fairness with the constraint:
\begin{equation}
  \label{eq:summary-example}
  \inner{w}{\bar{x}} \le c, \qquad \bar{x} = { \abs{D^M} }^{-1}
  \textstyle{\sum_{x \in D^M}} x \;-\; \abs{D^F}^{-1} \sum_{x \in D^F} x
  \eqperiod
\end{equation}
An SVM subject to this constraint (see \appref{summary-examples} for details),
for a range of $c$ values, is our baseline.

Results in \figref{fairness} show the proposed method is much more accurate for
any desired fairness, and achieves fairness ratios not reachable with the
approach of \citet{Zafar:2015} for any choice of $c$. It is also easier to
control: the values of $c$ in \citet{Zafar:2015} do not have a clear
interpretation, whereas $\kappa$ is an effective proxy for the fairness ratio.

\subsection{Churn}\label{sec:experiments:churn}

\input{figures/fig-churn}
Our second set of experiments demonstrates meeting real-world requirements on a
proprietary problem from \company: predicting whether a user interface element
should be shown to a user, based on a $31$-dimensional vector of informative
features, which is mapped to a roughly $30\,000$-dimensional feature vector via
a fixed kernel function $\Phi$. We train classifiers that are linear with
respect to $\Phi(x)$. We are given the currently-deployed model, and seek to
train a classifier that (i) has high accuracy, (ii) has no worse recall than
the deployed model, and (iii) has low churn \wrt the deployed model.

We are given three datasets, $D_1$, $D_2$ and $D_3$, consisting of $131\,840$,
$53\,877$ and $68\,892$ examples, respectively. The datasets $D_1$ and $D_2$
are hand-labeled, while $D_3$ is unlabeled. In addition, $D_1$ was chosen via
active sampling, while $D_2$ and $D_3$ are sampled \iid from the underlying
data distribution. For all three datasets, we split out $80\%$ for training and
reserved $20\%$ for testing.
We address the three goals in the proposed framework by simultaneously training
the classifier to minimize the number of errors on $D_1$ plus the number of
false positives on $D_2$, subject to the constraints that the recall on $D_2$
be at least as high as the deployed model's recall (we're essentially
performing Neyman-Pearson classification on $D_2$), and that the churn \wrt the
deployed model on $D_3$ be no larger than a given target parameter.

%We found that $31$-dimensional linear models were not capable of outperforming
%the deployed model. Instead, we use a fixed feature transformation $\Phi$ that
%maps each $x$ to a roughly $30\,000$-dimensional feature vector, and train
%classifiers that are linear with respect to $\Phi(x)$.

%Hence, we actually train much higher-dimensional models on
%$30\,000$-dimensional transformed features $\Phi(x)$, where $\Phi$ is a fixed
%feature transformation. Ref/details on $\Phi$ can be added if accepted, but
%really tangential to this paper.

These experiments use a proprietary \texttt{C++} implementation
%
%(to handle the larger-scale problem)
%
of \algref{cutting-plane-multipliers}, using the combined SDCA and cutting
plane approach of \appref{SVM} to implement the inner optimizations over $w$
and $b$, with the CutChooser helper functions being as described in
\apprefs{cutting-plane:maximization}{SVM:bias:minimization}. We performed $5$
iterations of the majorization-minimization procedure of
\algref{majorization-minimization}.

Our baseline is an unconstrained SVM that is thresholded after training to
achieve the desired recall, but makes no effort to minimize churn. We chose the
regularization parameter $\lambda$ using a power-of-$10$ grid search, found
that ${10}^{-7}$ was best for this baseline, and then used $\lambda =
{10}^{-7}$ for all experiments.

The plots in \figref{churn} show the achieved churn and error rates on the
training and testing sets for a range of churn constraint values (red and blue
curves), compared to the baseline thresholded SVM (green lines). When using
deterministic thresholding of the learned classifier (the blue curves, which
significantly outperformed randomized classification--the red curves), the
proposed method achieves lower churn and better accuracy for all targeted churn
rates, while also meeting the recall constraint.

%Maya cut this because the reader doesn't realize there is an "expected
%trade-off" between churn and accuracy and we don't have space to deal with
%that.
%
%We can also observe that, while the expected tradeoff between churn and
%accuracy is present, it is not particularly severe.

As expected, the empirical churn is extremely close to the targeted churn on
the training set when using randomized classification (red curve, top left
plot), but less so on the $20\%$ held-out test set (top right plot). We
hypothesize this disparity is due to overfitting, as the classifier has
$30\,000$ parameters, and $D_3$ is rather small (please see
\appref{generalization} for a discussion of the generalization performance of
our approach). However, except for the lowest targeted churn, the actual
classifier churn (blue curves) is substantially lower than the targeted
churn.
Compared to the thresholded SVM baseline, our approach significantly reduces
churn without paying an accuracy cost.

%% file: figures/fig-fairness.tex
\begin{figure*}[t]

\begin{center}

\begin{tabularx}{\textwidth}{cc}
  \includegraphics[width=\plotwidth]{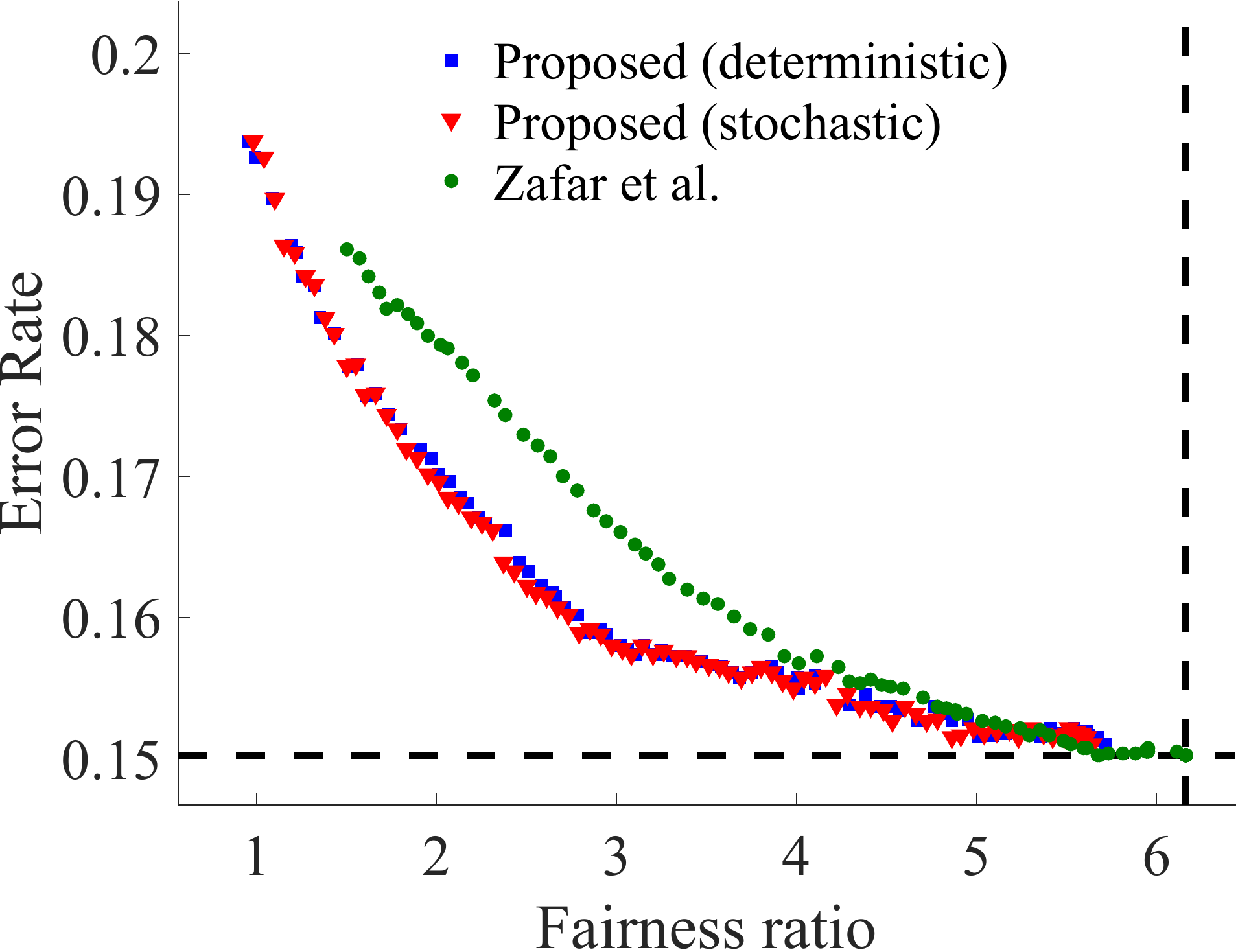} &
  \includegraphics[width=\plotwidth]{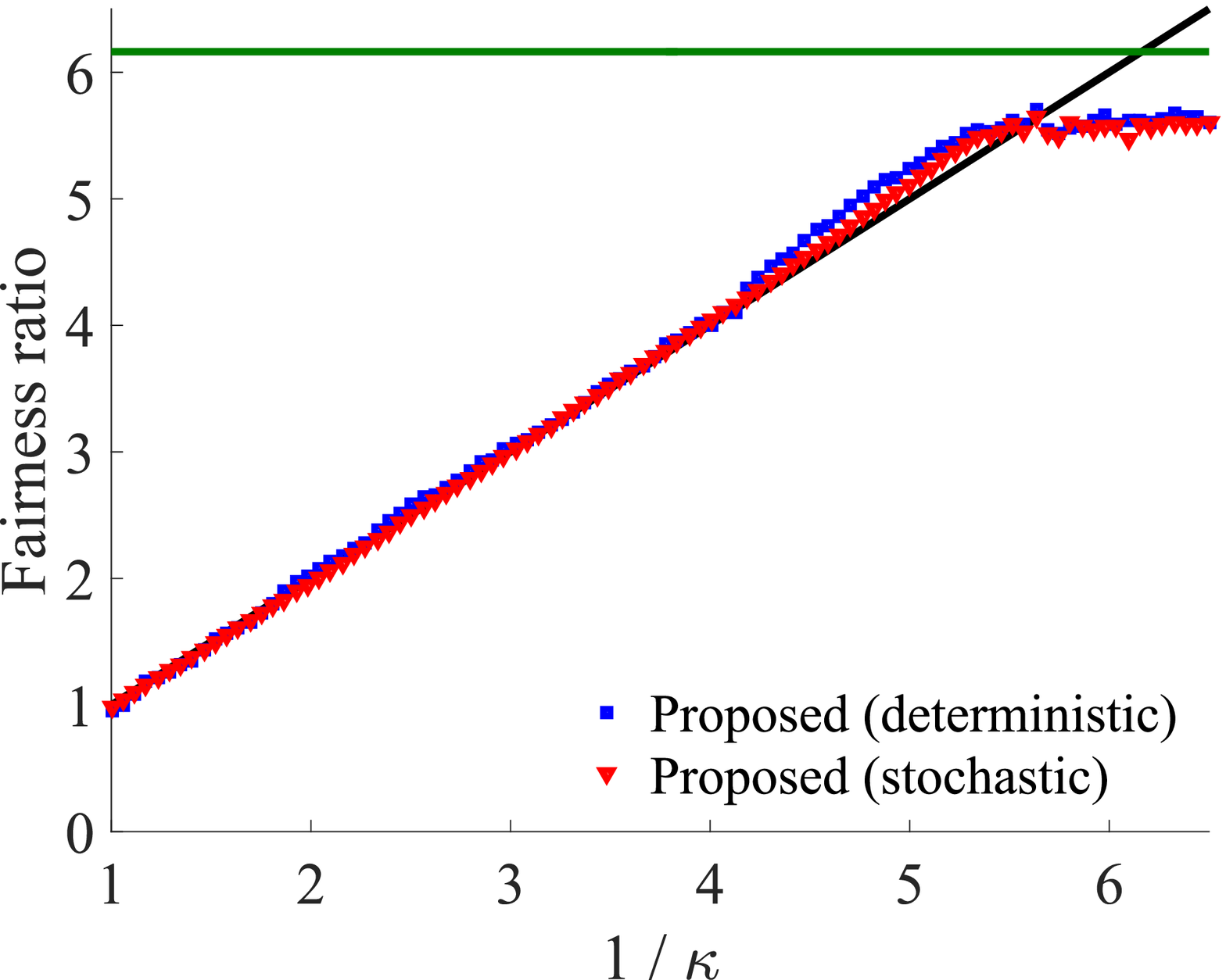}
\end{tabularx}

\end{center}

\caption{
  Blue dots: our proposal, with the classification functions' predictions being
  deterministically thresholded at zero.
  Red dots: same, but using the randomized classification rule described in
  \secref{problem}.
  Green dots: \citet{Zafar:2015}.
  Green line: unconstrained SVM.
  \textbf{(Left)} Test set error plotted vs. observed test set fairness ratio
  $\positiverate\left( D^M \right) / \positiverate\left( D^F \right)$.
  \textbf{(Right)} The $1/\kappa$ hyper-parameter used to specify the desired
  fairness in the proposed method, and the observed fairness ratios of our
  classifiers on the test data.
  All points are averaged over 100 runs.
}

\label{fig:fairness}

\end{figure*}

%% file: figures/fig-churn.tex
\begin{figure*}[t]

\begin{center}

\begin{tabularx}{\textwidth}{cc}
  \textbf{\large{Training}} & \textbf{\large{Testing}} \\
  \includegraphics[width=\plotwidth]{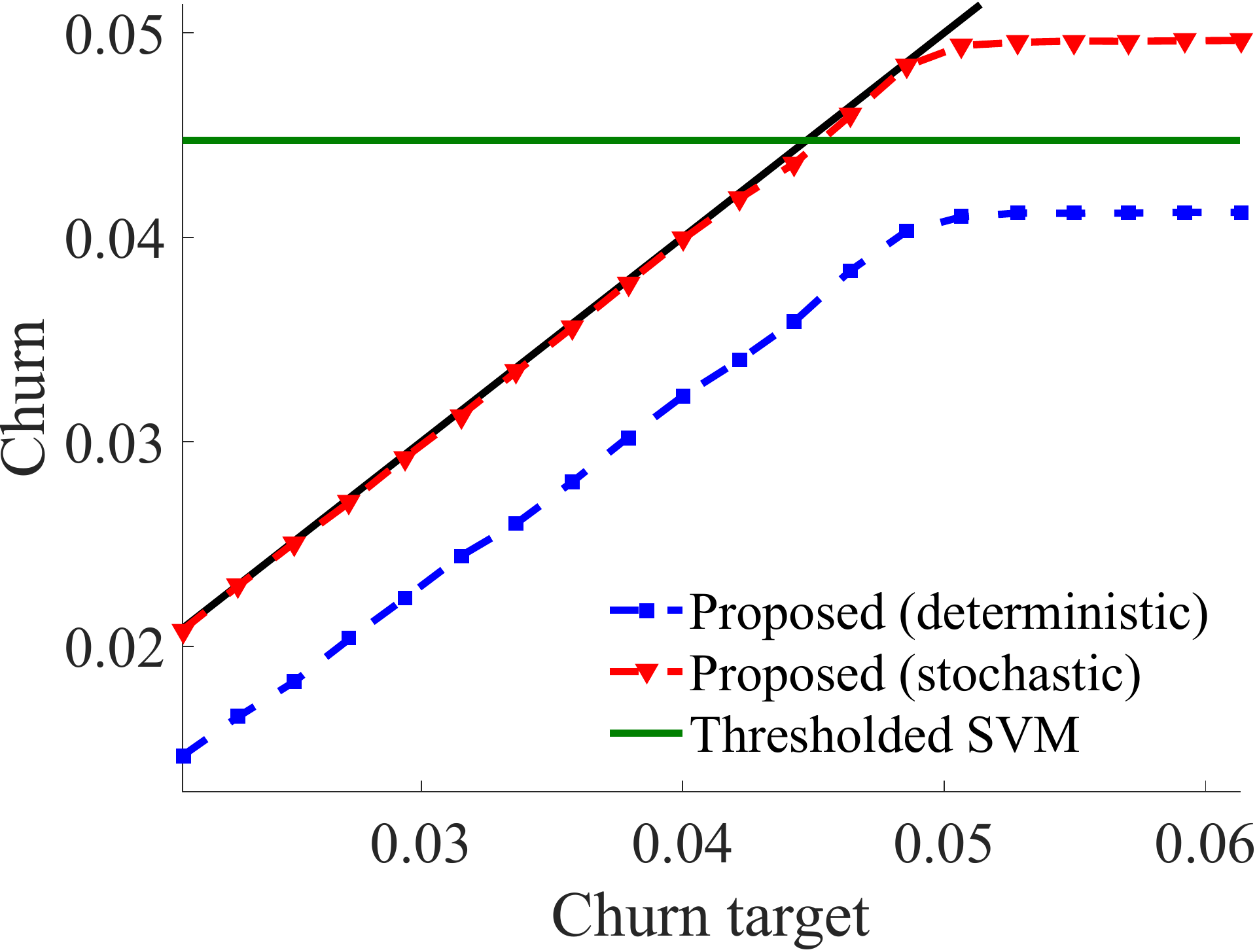} &
  \includegraphics[width=\plotwidth]{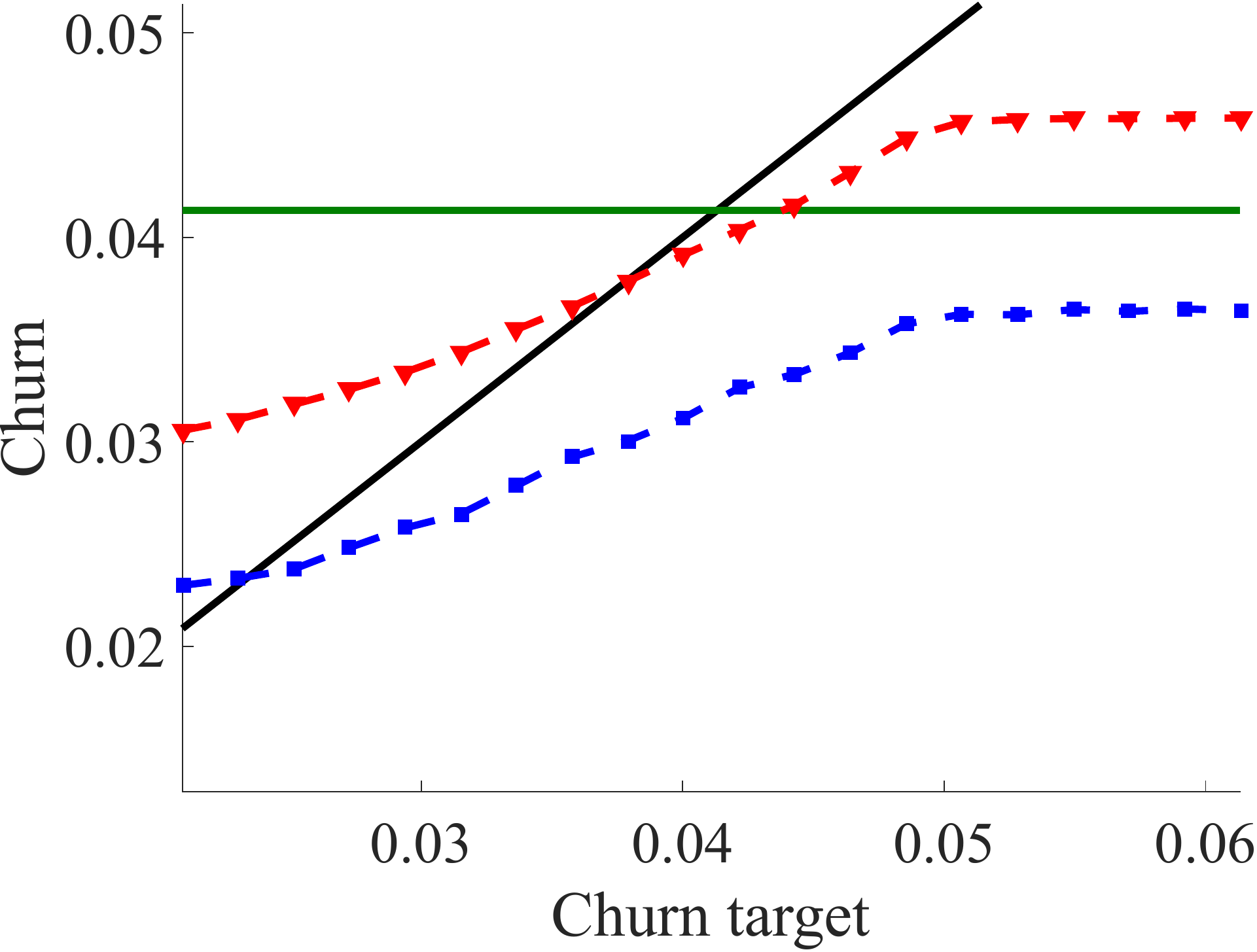} \\
  \includegraphics[width=\plotwidth]{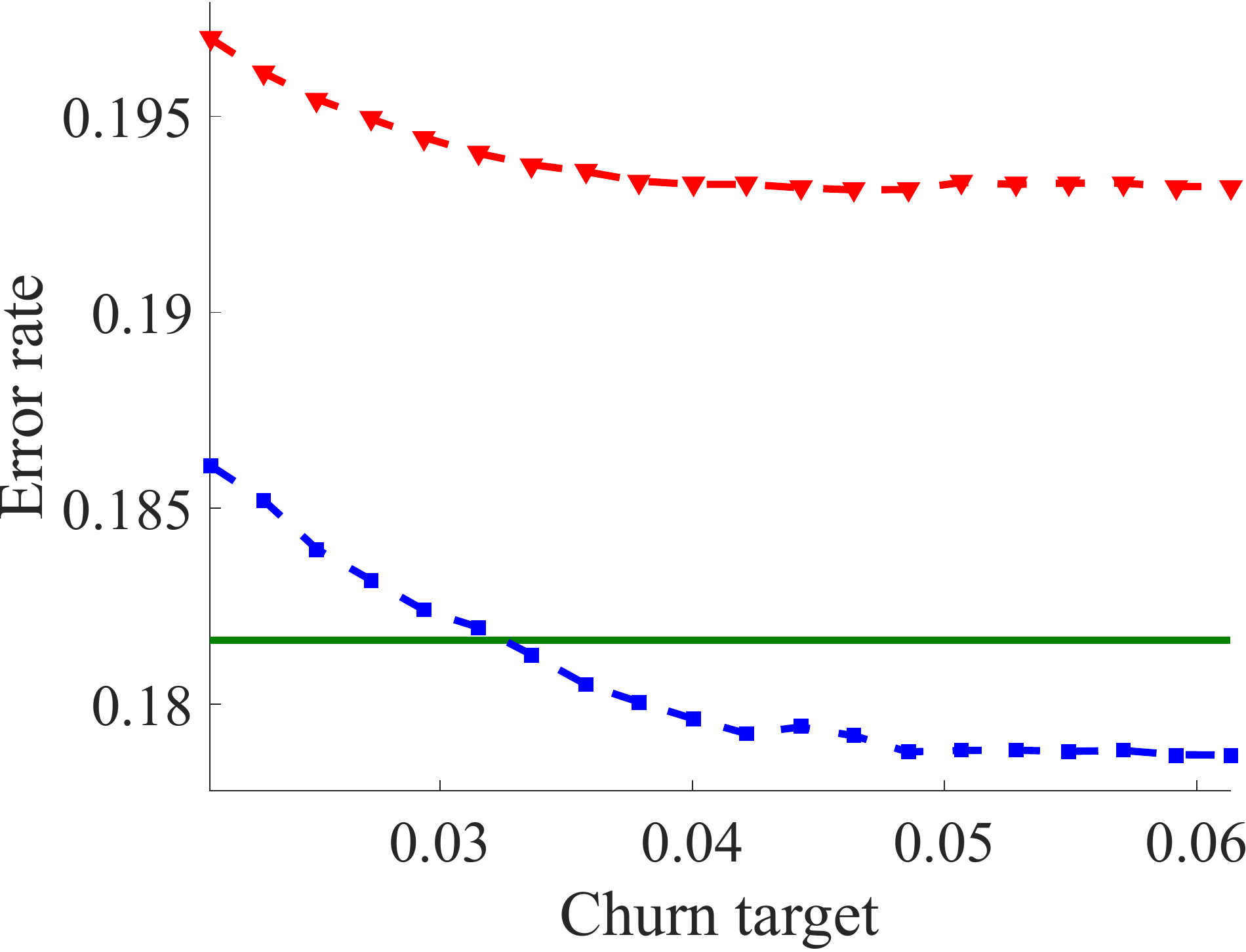} &
  \includegraphics[width=\plotwidth]{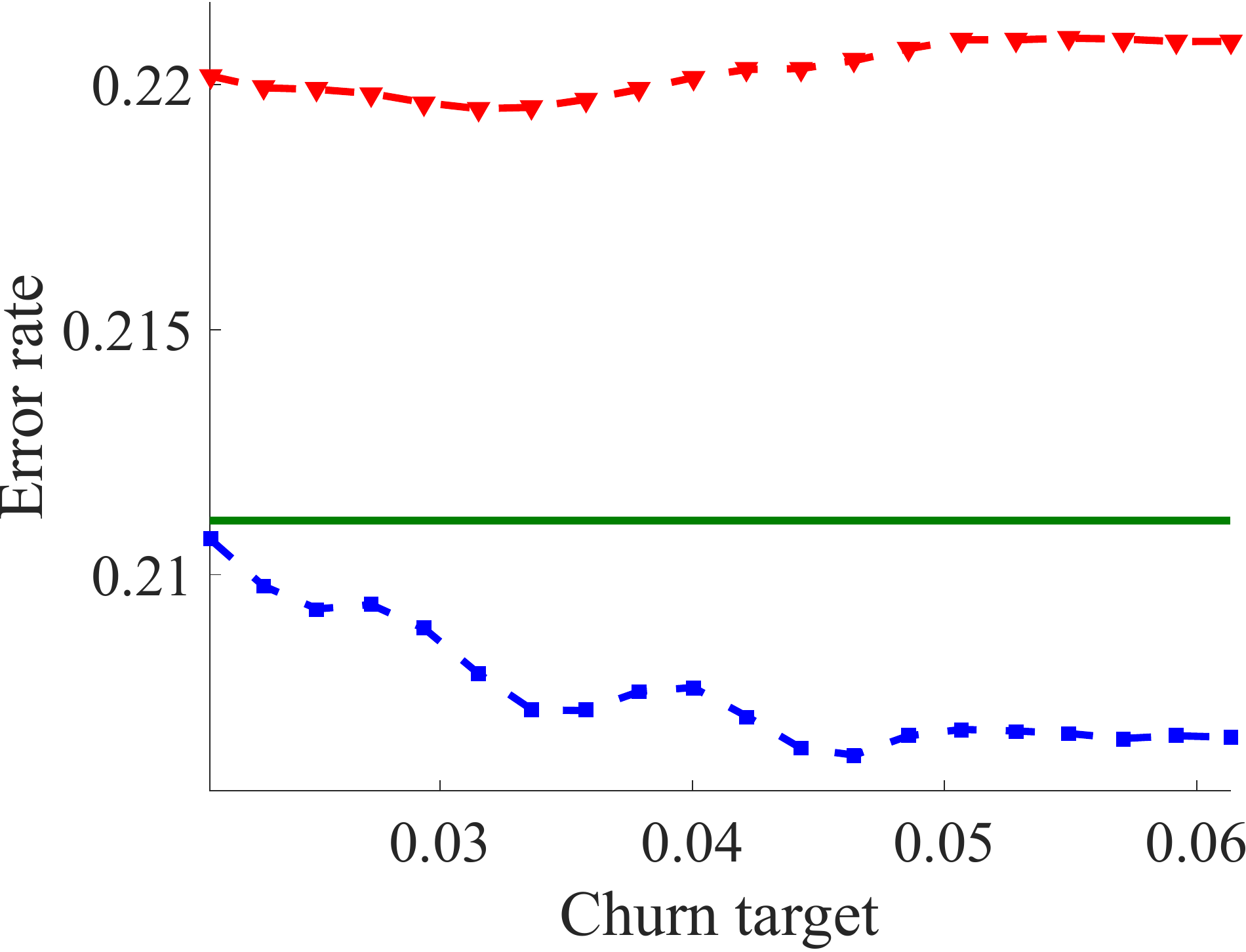}
\end{tabularx}

\end{center}

\caption{
  Blue: our proposal, with the classification functions' predictions being
  deterministically thresholded at zero.
  Red: same, but using the randomized classification rule described in
  \secref{problem}.
  Green: unconstrained SVM trained on $D_1 \cup D_2$, then thresholded (by
  shifting the bias $b$) to satisfy the recall constraint on $D_2$.
  \textbf{(Top)} Observed churn (vertical axis) vs. the churn target used
  during training (horizontal axis), on the train (left) and test (right)
  splits of the unlabeled dataset $D_3$.
  \textbf{(Bottom)} Empirical error rates (vertical axis) vs. the churn target,
  on the train (left) and test (right) splits of the union $D_1 \cup D_2$ of
  the two labeled datasets.
  All curves are averaged over 10 runs.
}

\label{fig:churn}

\end{figure*}

%% file: figures/tab-notation.tex
\begin{table*}[th!]

\caption{Key notation, listed in the order in which it was introduced.}

\label{tab:notation}

\begin{center}

\begin{tabular}{lll}
  \hline
  \textbf{Symbol} & \textbf{Introduced} & \textbf{Description} \\
  \hline
  $k$ & \secref{problem} & Number of datasets \\
  $m$ & \secref{problem} & Number of dataset constraints \\
  $D_i$ & \secref{problem} & $i$th dataset \\
  %
  % ----------------
  %
  $\positiverate$, $\negativerate$ & \secref{problem}, \eqref{rates} & Positive and negative indicator-based rates \\
  $\lambda$ & \secref{problem}, \probref{original-objective} & Regularization parameter \\
  $\positivelosscoefficient{i}$, $\negativelosscoefficient{i}$ & \secref{problem}, \probref{original-objective} & Coefficients defining the objective function \\
  $\positiveconstraintcoefficient{j}{i}$, $\negativeconstraintcoefficient{j}{i}$ & \secref{problem}, \probref{original-objective} & Coefficients defining the $j$th dataset constraint \\
  $\constraintbound{j}$ & \secref{problem}, \probref{original-objective} & Given upper bound of the $j$th dataset constraint \\
  $\ramp$ & \secref{problem}, \eqref{ramp-loss} & Ramp function: $\ramp(z) = \max \{ 0, \min\{ 1, \nicefrac{1}{2} + z \} \}$ \\
  $\ramppositiverate$, $\rampnegativerate$ & \secref{problem}, \eqref{ramp-rates} & Positive and negative ramp-based rates \\
  $\positiverampbound$, $\negativerampbound$ & \secref{problem:optimization}, \eqref{ramp-bounds} & Convex upper bounds on ramp functions \\
  $\ramppositiveratebound$, $\rampnegativeratebound$ & \secref{problem:optimization}, \eqref{rate-bounds} & Convex upper bounds on ramp-based rates \\
  %
  % ----------------
  %
  $\Psi$ & \secref{problem:cutting-plane}, \eqref{psi-definition} & SVM objective (for minimizing over $w$ and $b$) \\
  $\svmproblem$ & \secref{problem:cutting-plane}, \eqref{dual-problem} & Optimum of $\Psi$ (for maximizing over $\multipliers$) \\
  $\multipliers$ & \secref{problem:cutting-plane} & Lagrange multipliers associated with dataset constraints \\
  $\multiplierspace$ & \secref{problem:cutting-plane}, \algref{cutting-plane-multipliers} & Set of allowed $\multipliers$s \\
  %
  % ----------------
  %
  $\multipliers^{(s)}$ & \secref{problem:cutting-plane}, \algref{cutting-plane-multipliers} & Candidate solution at the $t$th iteration \\
  $l_t$, $u_t$ & \secref{problem:cutting-plane}, \algref{cutting-plane-multipliers} & Lower and upper bounds on $\svmproblem(\multipliers^{(t)})$ \\
  $g^{(t)}$ & \secref{problem:cutting-plane}, \algref{cutting-plane-multipliers} & Gradient of the cutting plane inserted at the $t$th iteration \\
  $h_t$ & \secref{problem:cutting-plane}, \algref{cutting-plane-multipliers} & Concave function upper-bounding $\svmproblem(\multipliers)$ \\
  $L_t$, $U_t$ & \secref{problem:cutting-plane}, \algref{cutting-plane-multipliers} & Lower and upper bounds on $\max_{\multipliers\in\multiplierspace} \svmproblem(\multipliers)$ \\
  %
  % ----------------
  %
  \hline
  %
  % ----------------
  %
  $\multiplierbound$ & \appref{generalization} & Maximum allowed $\multipliers_j$: $\multiplierspace \subseteq [0, \multiplierbound]^m$ \\
  $\expectedpositiverate$, $\expectednegativerate$ & \appref{generalization}, \eqref{expected-rates} & Expected positive and negative indicator-based rates \\
  %
  % ----------------
  %
  $\lebesguemeasure$ & \appref{cutting-plane} & Lebesgue measure \\
  $\superlevelset$ & \appref{cutting-plane:centroid}, \eqref{superlevel-set-definition} & Superlevel set \\
  $\superlevelhypograph$ & \appref{cutting-plane:centroid}, \eqref{superlevel-hypograph-definition} & Superlevel hypograph \\
  %
  % ----------------
  %
  $n$ & \appref{SVM:SDCA} & Total size of datasets: $n = \sum_{i=1}^k \abs{D_i}$ \\
  $\positivesvmlosscoefficient{i}$, $\negativesvmlosscoefficient{i}$ & \appref{SVM:SDCA}, \eqref{SVM-loss-coefficients} & Coefficients defining the convex objective function \\
  $\positivesvmconstraintcoefficient{j}{i}$, $\negativesvmconstraintcoefficient{j}{i}$ & \appref{SVM:SDCA}, \eqref{SVM-constraint-coefficients} & Coefficients defining the $j$th convex dataset constraint \\
  $\svmconstraintbound{j}$ & \appref{SVM:SDCA}, \eqref{SVM-constraint-bounds} & Given upper bound of the $j$th convex dataset constraint \\
  $\ell_{i,x}$ & \appref{SVM:SDCA}, \eqref{SVM-losses} & Loss of example $x$ in dataset $D_i$, in the SVM objective \\
  $\positivesvmcoefficient{i}$, $\negativesvmcoefficient{i}$ & \appref{SVM:SDCA}, \eqref{SVM-coefficients} & Coefficients defining the SVM objective function \\
  $L$ & \appref{SVM:SDCA}, \eqref{Lipschitz-definition} & Lipschitz constant of the $\ell_{i,x}$s \\
  $\dualvariables$ & \appref{SVM:SDCA}, \eqref{psi-dual-definition} & SVM dual variables \\
  $\Psi^*$ & \appref{SVM:SDCA}, \eqref{psi-dual-definition} & SVM dual objective (for maximizing over $\dualvariables$) \\
  %
  % ----------------
  %
  $b_s$ & \appref{SVM:bias}, \algref{cutting-plane-bias} & Candidate solution at the $t$th iteration \\
  $l_t'$, $u_t'$ & \appref{SVM:bias}, \algref{cutting-plane-bias} & Lower and upper bounds on $\min_{w \in \R^d} \Psi(w, b_t, \multipliers; w', b')$ \\
  $g_t'$ & \appref{SVM:bias}, \algref{cutting-plane-bias} & Derivative of the cutting plane inserted at the $t$th iteration \\
  $h_t'$ & \appref{SVM:bias}, \algref{cutting-plane-bias} & Convex function lower-bounding $\min_{w \in \R^d} \Psi(w, b, \multipliers; w', b')$ \\
  $L_t'$, $U_t'$ & \appref{SVM:bias}, \algref{cutting-plane-bias} & Lower and upper bounds on $\min_{b \in \mathcal{B}, w \in \R^d} \Psi(w, b, \multipliers; w', b')$ \\
  \hline
\end{tabular}

\end{center}

\end{table*}

%% file: app-randomized-classification.tex
\section{Randomized classification}\label{app:randomized-classification}

The use of the ramp loss in \probref{ramp-objective} can be interpreted in two
ways, which are exactly equivalent at training time, but lead to the use of
different classification rules at evaluation time.

\paragraph{Deterministic:}
This is the obvious interpretation: we would like to optimize
\probref{original-objective}, but cannot do so because the indicator-based
rates $\positiverate$ and $\negativerate$ are discontinuous, so we approximate
them with the ramp-based rates $\rampnegativerate$ and $\ramppositiverate$, and
and hope that this approximation doesn't cost us too much, in terms of
performance. The result is \probref{ramp-objective}. At evaluation time, on an
example $x$, we make a positive prediction if $\inner{w}{x} - b$ is
nonnegative, and a negative prediction otherwise.

\paragraph{Randomized:}
In this interpretation (also used by \citet{Cotter:2013}), we reinterpret the
ramp loss as the expected 0/1 loss suffered by a randomized classifier, with
the result that the rates aren't being approximated \emph{at all}---instead,
we're using the indicator-based rates throughout, but randomizing the
classifier and taking expectations to smooth out the discontinuities in the
objective function. To be precise, at evaluation time, on an example $x$, we
make a positive prediction with probability $\ramp(\inner{w}{x} - b)$, and a
negative prediction otherwise (with $\ramp$ being the ramp function of
\eqref{ramp-loss}). Taking expectations of the indicator-based rates
$\positiverate$ and $\negativerate$ over the randomness of this classification
rule yields the ramp-based rates $\rampnegativerate$ and $\ramppositiverate$,
resulting, once again, in \probref{ramp-objective}.

This use of a randomized prediction isn't as unfamiliar as it may at first
seem: in logistic regression, the classifier provides probability estimates at
evaluation time (with $\sigma$ being a sigmoid instead of a ramp). Furthermore,
at training time, the learned classifier is assumed to be randomized, so that
the optimization problem can be interpreted as maximizing the data
log-likelihood.

In the setting of this paper, the main advantages of the use of a randomized
classification rule are that (i) we can say something about generalization
performance (\appref{generalization}), and (ii) because the rates are never
being approximated, the dataset constraints will be satisfied \emph{tightly} on
the training dataset, in expectation (this is easily seen in the red curve in
the top left plot of \figref{churn}). Despite these apparent advantages,
deterministic classifiers seem to work better in practice.

%% file: app-ratio-metrics.tex
\section{Ratio metrics}\label{app:ratio-metrics}

\input{figures/tab-ratio-metrics}
\probref{original-objective} minimizes an objective function and imposes
upper-bound constraints, all of which are written as linear combinations of
positive and negative rates---we refer to such as ``linear combination
metrics''.
Some metrics of interest, however, cannot be written in this form. One
important subclass are the so-called ``ratio metrics'', which are \emph{ratios}
of linear combinations of rates. Examples of ratio metrics are precision,
$F_1$-score, win/loss ratio and win/change ratio (recall is a linear combination
metric, since its denominator is a constant).

Ratio metrics may not be used directly in the objective function, but can be
included in constraints by multiplying through by the denominator, then
shifting the constraint coefficients to be non-negative. For example, the
constraint that precision must be greater than $90\%$ can be expressed as
follows:
\begin{align*}
  \abs{D^+} \positiverate\left( D^+ \right) \geq & 0.9 \left( \abs{D^+}
  \positiverate\left( D^+ \right) + \abs{D^-} \positiverate\left( D^-
  \right)\right) \\
  0.1 \abs{D^+} \positiverate\left( D^+ \right) - 0.9 \abs{D^-}
  \positiverate\left( D^- \right) \geq & 0 \\
  -0.1 \abs{D^+} \positiverate\left( D^+ \right) + 0.9 \abs{D^-}
  \positiverate\left( D^- \right) \leq & 0 \\
  0.1 \abs{D^+} \negativerate\left( D^+ \right) + 0.9 \abs{D^-}
  \positiverate\left( D^- \right) \leq & 0.1 \abs{D^+} \eqcomma
\end{align*}
where we used the fact that $\positiverate\left( D^+ \right) +
\negativerate\left( D^+ \right) = 1$ on the last line---this is an example of a
fact that we noted in \secref{problem}: since positive and negative rates must
sum to one, it is possible to write any linear combination of rates as a
positive linear combination, plus a constant.

Multiplying through by the denominator is fine for
\probref{original-objective}, but a natural question is whether, by using a
randomized classifier and optimizing \probref{ramp-objective}, we're doing the
``right thing'' in expectation. The answer is: not quite. Since the expectation
of a ratio is not the ratio of expectations, \eg a precision constraint in our
original problem (\probref{original-objective}) becomes only a constraint on a
precision-like quantity (the ratio of the expectations of the precision's
numerator and denominator) in our relaxed problem.

%% file: figures/tab-ratio-metrics.tex
\begin{table*}[t]

\caption{
  Some ratio metrics (\appref{ratio-metrics}), which are metrics that can be
  written as ratios of linear combinations of rates. $\nwins$ and $\nlosses$
  are actually linear combination metrics, but are needed for the other
  definitions (as are Recall and $\nchanges$ from \tabref{metrics}).
}

\label{tab:ratio-metrics}

\begin{center}

\begin{tabularx}{\textwidth}{lX}
  \hline
  \textbf{Metric} & \textbf{Expression} \\
  \hline
  Precision & $\ntruepositives / \left(\ntruepositives +
  \nfalsepositives\right)$ \\
  $F_1$-score & $2 \mbox{Precision} \cdot \mbox{Recall} /
  \left(\mbox{Precision} + \mbox{Recall}\right) = 2 \ntruepositives / \left( 2
  \ntruepositives + \nfalsenegatives + \nfalsepositives \right) $ \\
  $\nwins$ & $\abs{D^{+-}} \positiverate\left(D^{+-}\right) + \abs{D^{-+}}
  \negativerate\left(D^{-+}\right)$ \\
  $\nlosses$ & $\abs{D^{++}} \negativerate\left(D^{+-}\right) + \abs{D^{--}}
  \positiverate\left(D^{-+}\right)$ \\
  Win/loss Ratio & $\nwins / \nlosses$ \\
  Win/change Ratio & $\nwins / \nchanges$ \\
  \hline
\end{tabularx}

\end{center}

\end{table*}

%% file: app-generalization.tex
\section{Generalization}\label{app:generalization}

In this appendix, we'll provide generalization bounds for an algorithm that is
\emph{nearly} identical to \algref{majorization-minimization}. The two
differences are that (i) we assume that the optimizer used on line 3 will
prefer smaller biases $b$ to larger ones, \ie that if
\probref{convex-objective} has multiple equivalent minima, then the optimizer
will return one for which $\abs{b}$ is minimized, and (ii) that the Lagrange
multipliers are upper-bounded by a parameter $\multiplierbound \ge
\multipliers_j$, \ie that instead of optimizing \eqref{dual-problem}, line 3 of
\algref{majorization-minimization} will optimize:
\begin{equation}
  \label{eq:generalization-dual-problem} \max_{0 \preceq \multipliers \preceq
  \multiplierbound} \min_{w,b} \Psi\left( w, b, \multipliers; w', b' \right)
  \eqcomma
\end{equation}
the difference being the upper bound on $\multipliers$. If $V$ is large enough
that no $v_j$s are bound to a constraint, then this will have no effect on the
solution. If, however, $V$ is too small, then the solution might not satisfy
the dataset constraints.
Notice that \algref{cutting-plane-multipliers} assumes that $\multipliers \in
\multiplierspace$, with $\multiplierspace$ being compact---hence, for our
proposed optimization procedure, the assumption is that $\multiplierspace
\subseteq \left[0, \multiplierbound\right]^m$.

With these assumptions in place, we're ready to move on to defining a function
class that contains any solution that could be found by our algorithm, and
bounding its Rademacher complexity.

\begin{lem}{generalization-function-class}
  Define $\mathcal{F}$ to be the set of all linear functions $f(x) =
  \inner{w}{x} - b$ with $\norm{w}_2 \le X B / \lambda$ and $\abs{b} \le 1/2 +
  X^2 B / \lambda$, where $X \ge \norm{x}_2$ is a uniform upper bound on the
  magnitudes of all training examples, and:
  \begin{equation*}
    B = \sum_{i=1}^k \left( \positivelosscoefficient{i} +
    \negativelosscoefficient{i} + \multiplierbound \sum_{j=1}^m \left(
    \positiveconstraintcoefficient{j}{i} + \negativeconstraintcoefficient{j}{i}
    \right) \right) \eqperiod
  \end{equation*}
  Then $\mathcal{F}$ will contain all $\abs{b}$-minimizing optimal solutions of
  \eqref{generalization-dual-problem} for any $(w', b')$ and any training
  dataset.
\end{lem}
\begin{proof}
  Let $f(w, b) + (\lambda / 2) \norm{w}_2^2$ be the the objective function of
  \probref{convex-objective}, and $g_j(w, b) \le \gamma^{(j)}$ the $j$th
  constraint. Then it follows that:
  \begin{align*}
    \norm{\nabla_w f\left(w,b\right)}_2 \le &
    X \sum_{i=1}^k \left( \positivelosscoefficient{i} +
    \negativelosscoefficient{i} \right) \\
    \norm{\nabla_w g_j\left(w,b\right)}_2 \le &
    X \multiplierbound \sum_{i=1}^k \left( \positiveconstraintcoefficient{j}{i}
    + \negativeconstraintcoefficient{j}{i} \right) \eqperiod
  \end{align*}
  Differentiating the definition of $\Psi$ (\eqref{psi-definition}) and setting
  the result equal to zero shows that any optimal $w$ must satisfy (this is the
  stationarity KKT condition):
  \begin{equation*}
    \lambda w = -\nabla_w f\left(w,b\right) - \sum_{j=1}^m \multipliers_j
    \nabla_w g_j\left(w,b\right) \eqcomma
  \end{equation*}
  implying by the triangle inequality that $\norm{w}_2 \le X B / \lambda$,
  where $B$ is as defined in the theorem statement.

  Now let's turn our attention to $b$. The above bound implies that, if $w$ is
  optimal, then $\abs{\inner{w}{x}} \le X^2 B / \lambda$, from which it follows
  that the hinge functions $\max\{ 0, \nicefrac{1}{2} + (\inner{w}{x} - b) \}$
  and $\max\{ 0, \nicefrac{1}{2} - (\inner{w}{x} - b) \}$ will be nondecreasing
  in $\abs{b}$ as long as $\abs{b} > 1/2 + X^2 B / \lambda$.
  \probref{convex-objective} seeks to minimize a positive linear combination of
  such hinge functions subject to upper-bound constraints on positive linear
  combinations of such hinge functions, so our assumption that the optimizer
  used on line 3 of \algref{majorization-minimization} will always choose the
  smallest optimal $b$ gives that $\abs{b} \le 1/2 + X^2 B / \lambda$.
\end{proof}

\begin{lem}{generalization-Rademacher}
  The function class $\mathcal{F}$ of \lemref{generalization-function-class}
  has Rademacher complexity~\citep{Bartlett:2002}:
  \begin{equation*}
    \rademacher_n \left(\mathcal{F}\right) \le \frac{1}{2\sqrt{n}} + \frac{2
    X^2}{\lambda \sqrt{n}} \sum_{i=1}^k \left( \positivelosscoefficient{i} +
    \negativelosscoefficient{i} + \multiplierbound \sum_{j=1}^m \left(
    \positiveconstraintcoefficient{j}{i} + \negativeconstraintcoefficient{j}{i}
    \right) \right) \eqcomma
  \end{equation*}
  where $X \ge \norm{x}_2$, as in \lemref{generalization-function-class}, is a
  uniform upper bound on the magnitudes of all training examples.
\end{lem}
\begin{proof}
  The Rademacher complexity of $\mathcal{F}$ is:
  \begin{align*}
    \rademacher_n \left(\mathcal{F}\right) =& \expectation \left[ \sup_{f \in
    \mathcal{F}} \frac{1}{n} \sum_{i=1}^n \epsilon_i f\left(x_i\right) \right]
    \\
    =& \expectation \left[ \sup_{w : \norm{w}_2 \le \frac{X B}{\lambda}}
    \frac{1}{n} \sum_{i=1}^n \epsilon_i \inner{w}{x} \right] + \expectation
    \left[ \sup_{b : \abs{b} \le \frac{1}{2} + \frac{X^2 B}{\lambda}}
    \frac{1}{n} \sum_{i=1}^n \epsilon_i b \right] \eqcomma
  \end{align*}
  where the expectations are taken over the \iid Rademacher random variables
  $\epsilon_1, \dots, \epsilon_n$ and the \iid training sample $x_1, \dots,
  x_n$, and $B$ is as in \lemref{generalization-function-class}. Applying the
  Khintchine inequality and substituting the definition of $B$ yields the
  claimed bound.
\end{proof}

We can now apply the results of \citet{Bartlett:2002} to prove bounds on the
generalization error. To this end, we assume that each of our training datasets
$D_i$ is drawn \iid from some underlying unknown distribution $\mathcal{D}_i$.
We will bound the expected positive and negative prediction rates \wrt these
distributions:
\begin{equation}
  \label{eq:expected-rates} \expectedpositiverate\left(\mathcal{D}; f\right) =
  \expectation_{x \sim \mathcal{D}} \left[ f\left(x\right) \right]
  \;\;\;\;\;\;\;\;
  \expectednegativerate\left(\mathcal{D}; f\right) =
  \expectedpositiverate\left(\mathcal{D}; 1 - f\right) \eqcomma
\end{equation}
where $f:\R^d \rightarrow \{0, 1\}$ is a binary classification function.

\begin{thm}{generalization-rate-bound}
  For a given $(w,b)$ pair, define $f_{w,b}(x)$ such that it predicts $1$ with
  probability $\ramp(\inner{w}{x} - b)$, and $0$ otherwise ($\ramp$ is as in
  \eqref{ramp-loss}, so this is the randomized classifier of
  \appref{randomized-classification}).

  Suppose that the $k$ training datasets $D_i$ have sizes $n_i = \abs{D_i}$,
  and that $D_i$ is drawn \iid from $\mathcal{D}_i$ for all $i \in \{1, \dots,
  k\}$. Then, with probability $1 - \delta$ over the training samples,
  uniformly over all $(w,b)$ pairs that are optimal solutions of
  \eqref{generalization-dual-problem} for some $(w',b')$ under the assumptions
  listed at the start of this appendix, the expected rates will satisfy:
  \begin{align*}
    \expectedpositiverate\left(\mathcal{D}_i; f_{w,b}\right) \le &
    \ramppositiverate\left(D_i;w,b\right) + E / \sqrt{n_i} \\
    \expectednegativerate\left(\mathcal{D}_i; f_{w,b}\right) \le &
    \rampnegativerate\left(D_i;w,b\right) + E / \sqrt{n_i} \eqcomma
  \end{align*}
  the above holding for all $i \in \{1, \dots, k\}$, where:
  \begin{equation*}
    E = 1 + \frac{4 X^2}{\lambda} \sum_{i=1}^k \left(
    \positivelosscoefficient{i} + \negativelosscoefficient{i} +
    \multiplierbound \sum_{j=1}^m \left( \positiveconstraintcoefficient{j}{i} +
    \negativeconstraintcoefficient{j}{i} \right) \right) + \sqrt{8
    \ln\left(\frac{4k}{\delta}\right)} \eqcomma
  \end{equation*}
  with $X \ge \norm{x}_2$, as in
  \lemrefs{generalization-function-class}{generalization-Rademacher}, being a
  uniform upper bound on the magnitudes of all training examples $x \sim
  \mathcal{D}_i$ for all $i \in \{1, \dots, k\}$.
\end{thm}
\begin{proof}
  Observe that the ramp rates $\ramppositiverate$ and $\rampnegativerate$ are
  $1$-Lipschitz. Applying Theorems 8 and 12 (part 4) of \citet{Bartlett:2002}
  gives that each of the following inequalities hold with probability $1 -
  \delta / 2k$, for all $i \in \{1, \dots, k\}$:
  \begin{align*}
    \expectation_{x \sim \mathcal{D}_i} \left[ \ramppositiverate\left(\{x\}; w,
    b\right) \right] \le & \ramppositiverate\left(D_i;w,b\right) + 2
    \rademacher_{n_i} \left(\mathcal{F}\right) + \sqrt{\frac{8}{n_i}
    \ln\left(\frac{4k}{\delta}\right)} \\
    \expectation_{x \sim \mathcal{D}_i} \left[ \rampnegativerate\left(\{x\}; w,
    b\right) \right] \le & \rampnegativerate\left(D_i;w,b\right) + 2
    \rademacher_{n_i} \left(\mathcal{F}\right) + \sqrt{\frac{8}{n_i}
    \ln\left(\frac{4k}{\delta}\right)} \eqcomma
  \end{align*}
  where $\rademacher_n (\mathcal{F})$ is as in
  \lemref{generalization-Rademacher}.
  The union bound implies that all $2k$ inequalities hold simultaneously with
  probability $1 - \delta$.
  The LHSs above are the expected ramp-based rates of a deterministic
  classifier, but as was explained in \appref{randomized-classification}, these
  are identical to the expected indicator-based rates of a randomized
  classifier, which is what is claimed.
\end{proof}

An immediate consequence of this result is that (with probability $1 - \delta$)
if $(w,b)$ suffers the training loss:
\begin{equation*}
  \hat{\mathcal{L}} = \sum_{i=1}^k \left( \positivelosscoefficient{i}
  \ramppositiverate ( D_i; w, b) + \negativelosscoefficient{i}
  \rampnegativerate ( D_i; w, b) \right) \eqcomma
\end{equation*}
then the expected loss on previously-unseen data (drawn \iid from the same
distributions) will be upper-bounded by:
\begin{equation*}
  \hat{\mathcal{L}} + E \sum_{i=1}^k \frac{\positivelosscoefficient{i} +
  \negativelosscoefficient{i}}{\sqrt{n_i}} \eqperiod
\end{equation*}
Likewise, if $(w,b)$ satisfies the constraint:
\begin{equation*}
  \sum_{i=1}^k \left( \positiveconstraintcoefficient{j}{i} \ramppositiverate (
  D_i; w, b) + \negativeconstraintcoefficient{j}{i} \rampnegativerate ( D_i; w,
  b ) \right) \leq \constraintbound{j} \eqcomma
\end{equation*}
then the corresponding rate constraint on previously-unseen data will be
violated by no more than:
\begin{equation*}
  E \sum_{i=1}^k \frac{\positiveconstraintcoefficient{j}{i} +
  \negativeconstraintcoefficient{j}{i}}{\sqrt{n_i}}
\end{equation*}
in expectation, where, here and above, $E$ is as in
\thmref{generalization-rate-bound}.

%% file: app-summary-examples.tex
\section{Fairness constraints of \citet{Zafar:2015}}\label{app:summary-examples}

The constraints of \citet{Zafar:2015} can be interpreted as a relaxation of the
constraint $-c \le \positiverate(D^{A};w) - \positiverate(D^{B};w) \leq c$
under the linear approximation
\begin{equation*}
  \positiverate(D;w,b) \approx\frac{1}{\abs{D}} \sum_{x\in D}(\inner{w}{x} - b)
  \eqcomma
\end{equation*}
giving:
\begin{equation*}
  \positiverate(D^{A};w,b) - \positiverate(D^{B};w,b)
  \approx \frac{1}{ \abs{D^A} }\sum_{x\in D^{A}}( \inner{w}{x} - b)-\frac{1}{
  \abs{D^B} }\sum_{x\in D^{B}}( \inner{w}{x} - b) \\
  = \inner{w}{\bar{x}} \eqcomma
\end{equation*}
where $\bar{x}$ is defined as in \eqref{summary-example}. We can therefore
implement the approach of \citet{Zafar:2015} within our framework by adding the
constraints:
\begin{align*}
  \inner{w}{\bar{x}} \leq c & \iff\max\{0, 1 - \inner{w}{\bar{x}} \}\leq c+1 \\
  c\leq \inner{w}{\bar{x}} & \iff\max\{0, 1 + \inner{w}{\bar{x}} \}\leq c+1
  \eqcomma
\end{align*}
and solving the hinge constrained optimization problem described in
\probref{convex-objective}. Going further, we could implement these constraints
as egregious examples using the constraint:
\begin{align*}
  \inner{w}{\bar{x}} \leq c \iff & \inner{w}{\frac{1}{4c}\bar{x}}
  \leq\frac{1}{4} \iff\frac{1}{2} + \inner{w}{ \frac{1}{4c} \bar{x}}
  \leq\frac{3}{4} \\
  & \iff\min\left\{ \max\left\{ \frac{1}{2} + \inner{w}{\frac{1}{4c}\bar{x}}, 0
  \right\}, 1 \right\} \leq \frac{3}{4}\iff r_{p}(\bar{x})\leq\frac{3}{4}
  \eqcomma
\end{align*}
permitting us to perform an analogue of their approximations in ramp form.

%% file: app-cutting-plane.tex
\section{Cutting plane algorithm}\label{app:cutting-plane}

We'll now discuss some variants of \algref{cutting-plane-multipliers}. We
assume that $\svmproblem(\multipliers)$ is the function that we wish to
maximize for $\multipliers \in \multiplierspace$, where:
\begin{enumerate*}
  \item $\multiplierspace \subseteq \R^m$ is compact and convex.
  \item $\svmproblem : \multiplierspace \rightarrow \R$ is concave.
  \item $\svmproblem$ has a (not necessarily unique) maximizer $\multipliers^*
  = \argmax_{ \multipliers \in \multiplierspace } \svmproblem\left(
  \multipliers \right)$.
\end{enumerate*}
For the purposes of \algref{cutting-plane-multipliers}, we would take
$\svmproblem$ to be as in \eqref{dual-problem}, but the same approach can be
applied more generally.

\subsection{Maximization-based}\label{app:cutting-plane:maximization}

We're primarily interested in proving convergence rates, and will do so in
\appref{cutting-plane:centroid}. With that said, there is one easy-to-implement
variant of \algref{cutting-plane-multipliers} for which we have not proved a
convergence rate, but that we use in some of our experiments due to its
simplicity:
\begin{defn}
  \label{def:maximization-multipliers}
  \textbf{(Maximization-based \algref{cutting-plane-multipliers})} CutChooser
  chooses $\multipliers^{(t)} =  \argmax_{\multipliers \in \multiplierspace}
  h_t\left( \multipliers \right)$ and $\epsilon_t = (U_t - L_t) / 2$.
\end{defn}
Observe that this $\multipliers^{(t)}$ can be found at the same time as $U_t$
is computed, since both result from optimization of the same linear program.
However, despite the ease of implementing this variant, we have not proved any
convergence rates about it.

\subsection{Center of mass-based}\label{app:cutting-plane:centroid}

We'll now discuss a variant of \algref{cutting-plane-multipliers} that chooses
$v^{(t)}$ and $\epsilon_t$ based on the center of mass of the ``superlevel
hypograph'' determined by $h_t$ and $L_t$, which we define as the intersection
of the hypograph of $h_t$ (the set of $m+1$-dimensional points $(\multipliers,
z)$ for which $z \le h_t(z)$) and the half-space containing all points
$(\multipliers, z)$ for which $z \ge L_t$.
Notice that, in the context of \algref{cutting-plane-multipliers}, the
superlevel hypograph defined by $h_t$ and $L_t$ corresponds to the set of pairs
of candidate maximizers and their possible function values at the $t$th
iteration.
Because this variant is based on finding a cut center in the $m+1$-dimensional
hypograph, rather than an $m$-dimensional level set (which is arguably more
typical), this is an instance of what \citet{Boyd:2011} call an ``epigraph
cutting plane method''.

Throughout this section, we will take $\lebesguemeasure$ to be the Lebesgue
measure (either $1$-dimensional, $m$-dimensional, or $m+1$-dimensional,
depending on context). We also must define some notation for dealing with
superlevel sets and hypographs. For a concave $f : \multiplierspace \rightarrow
\R$ and $y \in \R$, define:
\begin{equation}
  \label{eq:superlevel-set-definition} \superlevelset\left( f, y \right) =
  \left\{ \multipliers \in \multiplierspace \mid f\left( \multipliers \right)
  \ge y \right\}
\end{equation}
as the superlevel set of $f$ at $y$. Further define:
\begin{equation}
  \label{eq:superlevel-hypograph-definition} \superlevelhypograph\left( f, y
  \right) = \left\{ \left( \multipliers, z \right) \in \multiplierspace \times
  \R \mid f\left( \multipliers \right) \ge z \ge y \right\}
\end{equation}
as the superlevel hypograph of $f$ above $y$.
With these definitions in place, we're ready to explicitly state the center of
mass-based rule for the CutChooser function on line 8 of
\algref{cutting-plane-multipliers}:
\begin{defn}
  \label{def:centroid-multipliers}
  \textbf{(Center of mass-based \algref{cutting-plane-multipliers})} CutChooser
  takes $(\multipliers^{(t)}, z_t)$ to be the center of mass of
  $\superlevelhypograph(h_t, L_t)$, and chooses $\epsilon_t = (z_t - L_t) / 2$.
\end{defn}
Finding the center of mass of a polytope is a difficult problem in
general~\citep{Nemirovski:1994, Rademacher:2007}, so our convergence results
for this version of CutChooser are mostly of theoretical interest. With that
  said, for one dimensional problems (the setting of \appref{SVM:bias}) it may
  be implemented efficiently.
%
% Section 3.2 of the Nemirovski reference, above.

Our final bit of ``setup'' before getting to our results is to state two
classic theorems, plus a corollary, which will be needed for our proofs. The
first enables us to interpolate the areas of superlevel sets:

\begin{thm}{Brunn-Minkowski}
  Suppose that the superlevel sets of a concave $f : \multiplierspace
  \rightarrow \R$ at $y_1$ and $y_2$ are nonempty, and take $\gamma \in \left[
  0, 1 \right]$. Then:
  \begin{equation*}
    \left( \measureof{\lebesguemeasure}{ \superlevelset\left( f, \gamma y_{ 1 }
    + \left( 1 - \gamma \right) y_{ 2 } \right) } \right)^{ \nicefrac{ 1 }{ m }
    } \ge\gamma \left( \measureof{\lebesguemeasure}{ \superlevelset\left( f,
    y_{ 1 } \right) } \right)^{ \nicefrac{ 1 }{ m } } + \left( 1 - \gamma
    \right) \left( \measureof{\lebesguemeasure}{ \superlevelset\left( f, y_{ 2
    } \right) } \right)^{ \nicefrac{ 1 }{ m } } \eqperiod
  \end{equation*}
\end{thm}
\begin{proof}
  This is the Brunn-Minkowski inequality~\egcite{Ball:1997}.
\end{proof}

This theorem has the immediate useful corollary:

\begin{cor}{simple-volume-bounds}
  Suppose that $f : \multiplierspace \rightarrow \R$ is concave with a
  maximizer $\multipliers^* \in \multiplierspace$, and that $\delta \ge 0$.
  Then:
  \begin{equation*}
    \left( \frac{ \delta }{ m + 1 } \right) \measureof{\lebesguemeasure}{
    \superlevelset\left( f, f\left( \multipliers^* \right) - \delta \right) }
    \le \measureof{\lebesguemeasure}{ \superlevelhypograph\left( f, f\left(
    \multipliers^* \right) + \delta \right) } \le \delta
    \measureof{\lebesguemeasure}{ \superlevelset\left( f, f\left(
    \multipliers^* \right) - \delta \right) } \eqperiod
  \end{equation*}
\end{cor}
\begin{proof}
  By \thmref{Brunn-Minkowski} (lower-bounding the second term on the RHS by
  zero), for $0 \le z \le \delta$:
  \begin{equation*}
    \measureof{\lebesguemeasure}{ \superlevelset\left( f, f\left(
    \multipliers^* \right) - z \right) } \ge \left( \frac{ z }{ \delta }
    \right)^m \measureof{\lebesguemeasure}{ \superlevelset\left( f, f\left(
    \multipliers^* \right) - \delta \right) } \eqcomma
  \end{equation*}
  from which integrating $\measureof{\lebesguemeasure}{
  \superlevelhypograph\left( f, f\left( \multipliers^* \right) - \delta \right)
  } = \int_{ 0 }^{ \delta } \measureof{\lebesguemeasure}{ \superlevelset\left(
  f, f\left( \multipliers^* \right) - z \right) } m\lebesguemeasure(z)$ yields
  the claimed lower bound. The upper bound follows immediately from the fact
  that the superlevel sets shrink as $z$ increases (i.e.
  $\measureof{\lebesguemeasure}{ \superlevelset\left( f, z' \right) } \le
  \measureof{\lebesguemeasure}{ \superlevelset\left( f, z \right) }$ for $z'
  \ge z$).
\end{proof}

The second classic result enables us to bound how much ``progress'' is made by
a cut based on the center of mass of a superlevel hypograph:

\begin{thm}{Grunbaum}
  Suppose that $S \subseteq \R^m$ is a convex set. If we let $z \in S$ be the
  center of mass of $S$, then for any half-space $H \ni z$:
  \begin{equation*}
    \frac{ \measureof{\lebesguemeasure}{ S \cap H } }{
    \measureof{\lebesguemeasure}{ S } } \ge \left( \frac{ m }{ m + 1 }
    \right)^m \ge\frac{ 1 }{ e } \eqperiod
  \end{equation*}
\end{thm}
\begin{proof}
  This is Theorem 2 of \citet{Grunbaum:1960}.
\end{proof}

With the preliminaries out of the way, we're ready to move on to our first
result: bounding the volumes of the superlevel hypographs of our $h_t$s,
assuming that we base our cuts on the centers of mass of the superlevel
hypographs:

\begin{lem}{centroid-multipliers-Vt-bound}
  In the context of \algref{cutting-plane-multipliers}, suppose that we choose
  $\multipliers^{(t)}$ and $\epsilon_t$ as in \defref{centroid-multipliers}.
  Then:
  \begin{equation*}
    \measureof{\lebesguemeasure}{ \superlevelhypograph\left( h_{ t + 1 }, L_{ t
    + 1 } \right) } \le \left( 1 - \frac{ 1 }{ 2e } \right)
    \measureof{\lebesguemeasure}{ \superlevelhypograph\left( h_t, L_t \right) }
    \eqcomma
  \end{equation*}
  from which it follows that:
  \begin{equation*}
    \measureof{\lebesguemeasure}{ \superlevelhypograph\left( h_t, L_t \right) }
    \le \left( 1 - \frac{ 1 }{ 2e } \right)^{ t - 1 } \left( u_0 - l_0
    \right) \measureof{\lebesguemeasure}{\multiplierspace} \eqcomma
  \end{equation*}
  for all $t$.
\end{lem}
\begin{proof}
  We'll consider two cases: $u_t \le z_t$ and $u_t > z_t$, corresponding to
  making a ``deep'' or ``shallow'' cut, respectively.

  \emph{Deep cut case:} If $u_t \le z_t$, then the hyperplane $u_t +
  \inner{g^{(t)}}{\multipliers - \multipliers^{(t)}}$ passes below the center
  of mass of $\superlevelhypograph(h_t, L_t)$, implying by \thmref{Grunbaum}
  that:
  \begin{equation*}
    \measureof{\lebesguemeasure}{ \superlevelhypograph\left( h_{ t + 1 }, L_{ t
    + 1 } \right) } \le \measureof{\lebesguemeasure}{
    \superlevelhypograph\left( h_{ t + 1 }, L_t \right) } \le \left(1 -
    \frac{1}{e}\right) \measureof{\lebesguemeasure}{ \superlevelhypograph\left(
    h_t, L_t \right) } \eqperiod
  \end{equation*}

  \emph{Shallow cut case:} Now suppose that $u_t > z_t$. Applying
  \thmref{Grunbaum} to the level cut $\{(\multipliers, z) \mid z \le z_t\}$ at
  $z_t$:
  \begin{align*}
    \frac{1}{e} \measureof{\lebesguemeasure}{ \superlevelhypograph\left( h_t,
    L_t \right) } \le & \int_{ L_t }^{ z_t } \measureof{\lebesguemeasure}{
    \left\{ \multipliers \in \multiplierspace \mid h_t \left( \multipliers
    \right) \ge z \right\} } d\lebesguemeasure(z) \\
    \le & \int_{ L_t }^{ (z_t + L_t) / 2 } \measureof{\lebesguemeasure}{
    \left\{ \multipliers \in \multiplierspace \mid h_t \left( \multipliers
    \right) \ge z \right\} } d\lebesguemeasure(z) \\
    & + \int_{ (z_t + L_t) / 2 }^{ z_t } \measureof{\lebesguemeasure}{ \left\{
    \multipliers \in \multiplierspace \mid h_t \left( \multipliers \right) \ge
    z \right\} } d\lebesguemeasure(z) \eqperiod
  \end{align*}
  Since $h_t$ is concave, its superlevel sets shrink for larger $z$, so the
  first integral on the RHS above is larger than the second, implying that:
  \begin{equation*}
    \frac{1}{ 2 e } \measureof{\lebesguemeasure}{ \superlevelhypograph\left(
    h_t, L_t \right) } \le \int_{ L_t }^{ (z_t + L_t) / 2 }
    \measureof{\lebesguemeasure}{ \left\{ \multipliers \in \multiplierspace
    \mid h_t \left( \multipliers \right) \ge z \right\} } d\lebesguemeasure(z)
    \eqperiod
  \end{equation*}
  The fact that $\epsilon_t = (z_t - L_t) / 2$ implies that $l_t > (z_t + L_t)
  / 2$, so $L_{ t + 1 } > (z_t + L_t) / 2$, and:
  \begin{equation*}
    \frac{1}{ 2 e } \measureof{\lebesguemeasure}{ \superlevelhypograph\left(
    h_t, L_t \right) }
    \le \int_{ L_t }^{ L_{ t + 1 } } \measureof{\lebesguemeasure}{ \left\{
    \multipliers \in \multiplierspace \mid h_t \left( \multipliers \right) \ge
    z \right\} } d\lebesguemeasure(z) \eqcomma
  \end{equation*}
  showing that we will cut off at least a $1 / 2e$-proportion of the total
  volume, completing the proof of the first claim.

  The second claim follows immediately by iterating the first, and observing
  that $\measureof{\lebesguemeasure}{ \superlevelhypograph\left( h_1, L_1
  \right) } = \left( u_0 - l_0 \right)
  \measureof{\lebesguemeasure}{\multiplierspace}$.
\end{proof}

The above result shows that the volumes of the superlevel hypographs of the
$h_t$s shrink at an exponential rate. However, our actual stopping condition
(line 5 of \algref{cutting-plane-multipliers}) depends not on the volume, but
rather the ``height'' $U_t - L_t$, so we would prefer a bound on this height,
rather than the volume. We find such a bound in the (proof of the) following
lemma, which establishes how many iterations must elapse before the stopping
condition is satisfied:

\begin{lem}{centroid-multipliers-convergence-rate}
  In the context of \algref{cutting-plane-multipliers}, suppose that we choose
  $\multipliers^{(t)}$ and $\epsilon_t$ as in \defref{centroid-multipliers}.
  Then there is a iteration count $T_{\epsilon}$ satisfying:
  \begin{equation*}
    T_{\epsilon} = O\left( m \ln \left( \frac{u_0 - l_0}{\epsilon} \right) +
    \ln \left( \frac{ \measureof{\lebesguemeasure}{\multiplierspace} }{
    \measureof{\lebesguemeasure}{ \superlevelset\left( \svmproblem, l_0 \right) } }
    \right) \right) \eqcomma
  \end{equation*}
  such that, if $t \ge T_{\epsilon}$, then $U_t - L_t \le \epsilon$. Hence,
  \algref{cutting-plane-multipliers} will terminate after $T_{\epsilon}$
  iterations.
\end{lem}
\begin{proof}
  By \corref{simple-volume-bounds}:
  \begin{equation*}
    \measureof{\lebesguemeasure}{ \superlevelhypograph\left( h_t, L_t \right) }
    \ge \left( \frac{ U_t - L_t }{ m + 1 } \right)
    \measureof{\lebesguemeasure}{ \superlevelset\left( h_t, L_t \right) } \eqperiod
  \end{equation*}
  If $L_t \le \svmproblem\left(\multipliers^*\right) - \epsilon$, then
  $\measureof{\lebesguemeasure}{ \superlevelset\left( h_t, L_t \right) } \ge
  \measureof{\lebesguemeasure}{ \superlevelset\left( h_t,
  \svmproblem\left(\multipliers^*\right) - \epsilon \right) }$ because $h_t$ is
  concave.  If $L_t > \svmproblem\left(\multipliers^*\right) - \epsilon$, then by
  \thmref{Brunn-Minkowski}:
  \begin{equation*}
    \measureof{\lebesguemeasure}{ \superlevelset\left( h_t, L_t \right) } \ge
    \left( \frac{U_t - L_t}{U_t - \svmproblem\left(\multipliers^*\right) + \epsilon}
    \right)^m \measureof{\lebesguemeasure}{ \superlevelset\left( h_t,
    \svmproblem\left(\multipliers^*\right) - \epsilon \right) } \eqperiod
  \end{equation*}
  In either case, $L_t \le \svmproblem\left(\multipliers^*\right)$ by definition, and
  we'll assume that $U_t - L_t > \epsilon$ (this will lead to a contradiction),
  so:
  \begin{equation*}
    \measureof{\lebesguemeasure}{ \superlevelhypograph\left( h_t, L_t \right) }
    \ge 2^{-m} \left( \frac{ U_t - L_t }{ m + 1 } \right)
    \measureof{\lebesguemeasure}{ \superlevelset\left( h_t,
    \svmproblem\left(\multipliers^*\right) - \epsilon \right) } \eqperiod
  \end{equation*}
  Applying \lemref{centroid-multipliers-Vt-bound} yields that:
  \begin{equation*}
    \left( 1 - \frac{ 1 }{ 2e } \right)^{ t - 1 } \left( u_0 - l_0 \right)
    \measureof{\lebesguemeasure}{\multiplierspace} \ge 2^{-m} \left( \frac{ U_t
    - L_t }{ m + 1 } \right) \measureof{\lebesguemeasure}{ \superlevelset\left(
    h_t, \svmproblem\left(\multipliers^*\right) - \epsilon \right) } \eqperiod
  \end{equation*}
  Next observe that, by \thmref{Brunn-Minkowski}:
  \begin{equation*}
    \measureof{\lebesguemeasure}{ \superlevelset\left( h_t,
    \svmproblem\left(\multipliers^*\right) - \epsilon \right) }
    \ge \left( \frac{ U_t - \svmproblem\left(\multipliers^*\right) + \epsilon }{ U_t -
    l_0 } \right)^m \measureof{\lebesguemeasure}{ \superlevelset\left( h_t, l_0
    \right) }
    \ge \left( \frac{ \epsilon }{ u_0 - l_0 } \right)^m
    \measureof{\lebesguemeasure}{ \superlevelset\left( \svmproblem, l_0 \right) } \eqperiod
  \end{equation*}
  Combining the previous two equations gives:
  \begin{equation*}
    U_t - L_t \le \left( 1 - \frac{ 1 }{ 2e } \right)^{ t - 1 } \left( m + 1
    \right) \left( \frac{ 2 }{ \epsilon } \right)^m \left( u_0 - l_0 \right)^{
    m + 1 } \left( \frac{ \measureof{\lebesguemeasure}{\multiplierspace} }{
    \measureof{\lebesguemeasure}{ \superlevelset\left( \svmproblem, l_0 \right) } }
    \right) \eqperiod
  \end{equation*}
  Simplifying this inequality yields that, if we have performed the claimed
  number of iterations, then $U_t - L_t \le \epsilon$ (this contradicts our
  earlier assumption that $U_t - L_t > \epsilon$, so this is technically a
  proof by contradiction).
\end{proof}

The second term in the bound on $T_{\epsilon}$ measures how closely
$\multiplierspace$ matches with the set of all points $z$ on which $\svmproblem\left(
z \right)$ exceeds our initial lower bound $l_0$.
Observe that if $l_0 \le \svmproblem\left( \multipliers \right)$ for all $\multipliers
\in \multiplierspace$, then $\measureof{\lebesguemeasure}{ \superlevelset\left(
\svmproblem, l_0 \right) } = \measureof{\lebesguemeasure}{\multiplierspace}$, and this
term will vanish.

Bounding the number of cutting-plane iterations that will be performed is not
enough to establish how quickly our procedure will converge, since we rely on
performing an inner SVM optimizations with target suboptimality $\epsilon_t$,
and the runtime of these inner optimizations naturally will depend on the
magnitudes of the $\epsilon_t$s, which are bounded in our final lemma:

\begin{lem}{centroid-multipliers-epsilon-bound}
  In the context of \algref{cutting-plane-multipliers}, suppose that we choose
  $\multipliers^{(t)}$ and $\epsilon_t$ as in \defref{centroid-multipliers}.
  Then:
  \begin{equation*}
    \epsilon_t \ge \frac{ U_t - L_t }{ 2e \left( m + 1 \right) } \eqcomma
  \end{equation*}
  and in particular, for all $t$ (before termination):
  \begin{equation*}
    \epsilon_t \ge \frac{ \epsilon }{ 2e \left( m + 1 \right) } \eqcomma
  \end{equation*}
  since we terminate as soon as $U_t - L_t \le \epsilon$.
\end{lem}
\begin{proof}
  Because $h_t$ is concave:
  \begin{equation*}
    \measureof{\lebesguemeasure}{ \superlevelhypograph\left( h_t, L_t \right) }
    - \measureof{\lebesguemeasure}{ \superlevelhypograph\left( h_t, z_t \right)
    } \le \left( z_t - L_t \right) \measureof{\lebesguemeasure}{
    \superlevelset\left( h_t, L_t \right) } \eqcomma
  \end{equation*}
  where $z_t$ is as in \lemref{centroid-multipliers-Vt-bound}.
  By \corref{simple-volume-bounds}, $\measureof{\lebesguemeasure}{
  \superlevelset\left( h_t, L_t \right) } \le \frac{ m + 1 }{ U_t - L_t }
  \measureof{\lebesguemeasure}{ \superlevelhypograph\left( h_t, L_t \right) }$,
  which combined with the above inequality gives that:
  \begin{equation*}
    \frac{ \measureof{\lebesguemeasure}{ \superlevelhypograph\left( h_t, L_t
    \right) } - \measureof{\lebesguemeasure}{ \superlevelhypograph\left( h_t,
    z_t \right) } }{ \measureof{\lebesguemeasure}{ \superlevelhypograph\left(
    h_t, L_t \right) } } \le \frac{ z_t - L_t }{ U_t - L_t } \left( m + 1
    \right) \eqperiod
  \end{equation*}
  By \thmref{Grunbaum}, the LHS is at least $\nicefrac{ 1 }{ e }$, and $z_t -
  L_t = 2\epsilon_t$, giving the claimed result.
\end{proof}

%% file: app-svm.tex
\section{SVM optimization}\label{app:SVM}

We'll now move onto a discussion of how we propose implementing the
SVMOptimizer of \algref{cutting-plane-multipliers}.
The easier-to-analyze approach, based on an inner SDCA optimization over
$w$~\citep{ShalevShwartz:2013} and an outer cutting plane optimization over $b$
(\algref{cutting-plane-bias}), will be described in
\apprefs{SVM:SDCA}{SVM:bias}.
The easier-to-implement version, which simply calls an off-the-shelf SVM
solver, will be described in \appref{SVM:kernel}.

\subsection{SDCA $w$-optimization}\label{app:SVM:SDCA}

To simplify the presentation, we're going to begin by reformulating
\eqref{psi-definition} in such a way that all of the datasets are ``mashed
together'', with the coefficients being defined on a per-example basis, rather
than per-dataset. To this end, for fixed $w'$ and $b'$, we define, for every $i
\in \{1, \dots, k\}$ and every $x \in D_i$:
\begin{align}
  \label{eq:SVM-loss-coefficients} \positivesvmlosscoefficient{i,x} = &
  \begin{cases}
    \positivelosscoefficient{i} & \;\;\;\;\mbox{if } \inner{ w' }{ x } - b' \le
    \nicefrac{1}{2} \\
    0 & \;\;\;\;\mbox{otherwise}
  \end{cases} \\
  \notag \negativesvmlosscoefficient{i,x} = & \begin{cases}
    \negativelosscoefficient{i} & \;\;\;\;\mbox{if } \inner{ w' }{ x } - b' \ge
    -\nicefrac{1}{2} \\
    0 & \;\;\;\;\mbox{otherwise}
  \end{cases} \eqperiod
\end{align}
This takes care of the loss coefficients. For the constraint coefficients,
define:
\begin{align}
  \label{eq:SVM-constraint-coefficients}
  \positivesvmconstraintcoefficient{j}{i,x} = & \begin{cases}
    \positiveconstraintcoefficient{j}{i} & \;\;\;\;\mbox{if } \inner{ w' }{ x }
    - b' \le \nicefrac{1}{2} \\
    0 & \;\;\;\;\mbox{otherwise}
  \end{cases} \\
  \notag \negativesvmconstraintcoefficient{j}{i,x} = & \begin{cases}
    \negativeconstraintcoefficient{j}{i} & \;\;\;\;\mbox{if } \inner{ w' }{ x }
    - b' \ge -\nicefrac{1}{2} \\
    0 & \;\;\;\;\mbox{otherwise}
  \end{cases} \eqperiod
\end{align}
and finally, we need to handle the constraint upper bounds:
\begin{align}
  \label{eq:SVM-constraint-bounds} \svmconstraintbound{j} = &
  \constraintbound{j} - \sum_{i=1}^{k} \frac{1}{\abs{D_i}} \left(
  \positiveconstraintcoefficient{j}{i} \abs{ \left\{ x \in D_i \mid \inner{ w'
  }{ x } - b' > \nicefrac{1}{2} \right\} } \right.  \\
  \notag & \left. + \negativeconstraintcoefficient{j}{i} \abs{ \left\{ x \in
  D_i \mid \inner{ w' }{ x } - b' < -\nicefrac{1}{2} \right\} } \right)
  \eqperiod
\end{align}
Observe that the $\positivesvmlosscoefficient{i,x}$s,
$\negativesvmlosscoefficient{i,x}$s,
$\positivesvmconstraintcoefficient{j}{i,x}$s,
$\negativesvmconstraintcoefficient{j}{i,x}$s, and $\svmconstraintbound{j}$s all
have implicit dependencies on $w'$ and $b'$.
In terms of these definitions, the $\Psi$ defined in \eqref{psi-definition} can
be written as:
\begin{align*}
  \Psi\left(w,b,\multipliers;w',b'\right) =
  & \sum_{i=1}^{k} \frac{1}{\abs{D_i}} \sum_{x \in D_i} \left( \left(
  \positivesvmlosscoefficient{i,x} + \sum_{j=1}^m \multipliers_j
  \positivesvmconstraintcoefficient{j}{i,x} \right) \max \left\{ 0, \frac{1}{2}
  + \left( \inner{w}{x} - b \right) \right\} \right. \\
  & \left. + \left( \negativesvmlosscoefficient{i,x} + \sum_{j=1}^m
  \multipliers_j \negativesvmconstraintcoefficient{j}{i,x} \right) \max \left\{
    0, \frac{1}{2} - \left( \inner{w}{x} - b \right) \right\} \right) \\
  & + \frac{\lambda}{2} \norm{w}_2^2 - \sum_{j=1}^m \multipliers_j
  \svmconstraintbound{j} \eqperiod
\end{align*}
This formulation makes it clear that minimizing $\Psi$ as a function of $w$ and
$b$ is equivalent to optimizing an SVM, since $\Psi$ is just a positive linear
combination of hinge losses, plus a $\ell^2$ regularizer, plus a term that does
not depend on $w$ or $b$. Since $\Psi$ can have both ``positive'' and
``negative'' hinge losses associated with the same example, however, it's
slightly simpler to combine both hinge losses together into a single piecewise
linear per-example loss, rather than decomposing it into two separate hinges:
\begin{equation}
  \label{eq:SVM-losses} \ell_{i,x}\left( z \right) =
  \positivesvmcoefficient{i,x} \max \left\{ 0, \frac{1}{2} + z \right\}
  + \negativesvmcoefficient{i,x} \max \left\{ 0, \frac{1}{2} - z \right\}
  \eqcomma
\end{equation}
where:
\begin{equation}
  \label{eq:SVM-coefficients} \positivesvmcoefficient{i,x} =
  \frac{n}{\abs{D_i}} \left( \positivesvmlosscoefficient{i,x} + \sum_{j=1}^m
  \multipliers_j \positivesvmconstraintcoefficient{j}{i,x} \right)
  \;\;\;\;\mbox{ and }\;\;\;\;
  \negativesvmcoefficient{i,x} = \frac{n}{\abs{D_i}} \left(
  \negativesvmlosscoefficient{i,x} + \sum_{j=1}^m \multipliers_j
  \negativesvmconstraintcoefficient{j}{i,x} \right) \eqperiod
\end{equation}
Here, $n = \sum_{i=1}^k \abs{D_i}$ is the total number of examples across all
of the datasets---we introduced the $n$ factor here so that $\Psi$ will be
written in terms of the \emph{average} loss (rather than the \emph{total}
loss).
Although it is not represented explicitly in our notation, it should be
emphasized that $\ell_{i,x}$ implicitly depends on $\multipliers$, $w'$ and
$b'$.

As the sum of two hinges, the $\ell_{i,x}$s are Lipschitz continuous in $z$,
with the Lipschitz constant being:
\begin{equation}
  \label{eq:Lipschitz-definition} L = \max_{i \in \{1, \dots, k\}}
  \frac{n}{\abs{D_i}} \left( \left( \positivelosscoefficient{i} +
  \negativelosscoefficient{i} \right) + \sum_{j=1}^m \multipliers_j \left(
  \positiveconstraintcoefficient{j}{i} + \negativeconstraintcoefficient{j}{i}
  \right) \right) \eqperiod
\end{equation}
Notice that, if the datasets are comparable in size, then $n / \abs{D_i}$ will
be on the order of $k$, so $L$ will typically not be as large as the
$n$-dependence of its definition would appear to imply.

We may now write $\Psi$ in terms of the loss functions $\ell_{i,x}$:
\begin{equation*}
  \Psi\left(w,b,\multipliers;w',b'\right) =
  \frac{1}{n} \sum_{i=1}^{k} \sum_{x \in D_i} \ell_{i,x}\left( \inner{w}{x} - b
  \right) + \frac{\lambda}{2} \norm{w}_2^2 - \sum_{j=1}^m \multipliers_j
  \svmconstraintbound{j} \eqperiod
\end{equation*}
This is the form considered by \citet{ShalevShwartz:2013}, so we may apply
SDCA:

\begin{thm}{SDCA}
  If we use SDCA~\citep{ShalevShwartz:2013} to optimize
  \eqref{psi-dual-definition} for fixed $b$ and $\multipliers$, then we will
  find a suboptimal solution with duality gap $\epsilon''$ after performing at
  most:
  \begin{equation*}
    T_{\epsilon''} = O\left( \max\left\{ 0, n \ln \left( \frac{\lambda n}{L^2
    X^2} \right) \right\} + n + \frac{L^2 X^2}{\lambda \epsilon''} \right)
  \end{equation*}
  iterations, where $X = \max_{i \in \{1, \dots, k\}} \max_{x \in D_i}
  \norm{x}_2$ is a uniform upper bound on the norms of the training examples.
\end{thm}
\begin{proof}
  This is Theorem 2 of \citet{ShalevShwartz:2013}.
\end{proof}

SDCA works by, rather than directly minimizing $\Psi$ over $w$, instead
maximizing the following over the dual variables $\dualvariables$:
\begin{align}
  \label{eq:psi-dual-definition} \MoveEqLeft
  \Psi^*\left(\dualvariables,b,\multipliers;w',b'\right) = \\
  \notag & - \frac{1}{n} \sum_{i=1}^{k} \sum_{x \in D_i}
  \ell^*_{i,x}\left(\dualvariables_{i,x}\right)
  - \frac{1}{2 \lambda} \norm{ \frac{1}{n} \sum_{i=1}^{k} \sum_{x \in D_i}
  \dualvariables_{i,x} x }_2^2
  - \frac{1}{n} \sum_{i=1}^{k} \sum_{x \in D_i} \dualvariables_{i,x} b
  - \sum_{j=1}^m \multipliers_j \svmconstraintbound{j} \eqcomma
\end{align}
using stochastic coordinate ascent, where:
\begin{equation*}
  w = -\frac{1}{\lambda n} \sum_{i=1}^{k} \sum_{x \in D_i} \dualvariables_{i,x}
  x
\end{equation*}
is the primal solution $w$ corresponding to a given set of dual variables
$\dualvariables$, and:
\begin{equation*}
  \ell^*_{i,x}\left( \dualvariables_{i,x} \right) = \frac{1}{2} \abs{
  \dualvariables_{i,x} - \positivesvmcoefficient{i,x} +
  \negativesvmcoefficient{i,x} } - \frac{1}{2} \left(
  \positivesvmcoefficient{i,x} + \negativesvmcoefficient{i,x} \right)
\end{equation*}
is the Fenchel conjugate of $\ell_{i,x}$, and is defined for
$-\negativesvmcoefficient{i,x} \le \dualvariables_{i,x} \le
\positivesvmcoefficient{i,x}$ (these bounds become box constraints on the
$\dualvariables$s of \eqref{psi-dual-definition}).

\subsection{Cutting plane $b$-optimization}\label{app:SVM:bias}

\input{figures/alg-cutting-plane-bias}
Having described in the previous section how we may optimize over $w$ for fixed
$b$ and $\multipliers$ using SDCA, we now move on to the problem of creating
the SVMOptimizer needed by \algref{cutting-plane-multipliers}, which must
optimize over both $w$ and $b$.

Many linear SVM optimizers do not natively handle an unregularized bias
parameter $b$, and this has long been recognized as a potential issue. For
example, \citet{ShalevShwartz:2010} suggest using Pegasos to perform inner
optimizations over $w$, and a bisection-based outer optimization over $b$. Our
proposal is basically this, except that \algref{cutting-plane-bias}, rather
than using bisection, optimizes over $b$ using essentially the same cutting
plane algorithm as we used in \algref{cutting-plane-multipliers}, except that
optimizing over $b$ is a minimization problem (over $\multipliers$ is
maximization), and we might increase $u_0'$ on line 2 of
\algref{cutting-plane-bias} for a technical reason (it will be needed by the
proof of \lemref{centroid-bias-convergence-rate}, but is probably not helpful
in practice).

\subsubsection{Minimization-based}\label{app:SVM:bias:minimization}

Perhaps the easiest-to-implement version of \algref{cutting-plane-bias} is
based on the idea of simply solving for the minimizer of $h_t'$ at every
iteration.
\begin{defn}
  \label{def:minimization-bias}
  \textbf{(Minimization-based \algref{cutting-plane-bias})} Do \emph{not}
  increase $u_0'$ on line 2, and have CutChooser choose $b_t = \argmin_{b \in
  \mathcal{B}} h_t'\left( b \right)$ and $\epsilon_t' =  (U_t - L_t) / 2$.
\end{defn}
As was the case in \appref{cutting-plane:maximization}, we have no convergence
rates for this version. Furthermore, since this is a one-dimensional problem,
the center of mass-based version of \algref{cutting-plane-bias} is
implementable and efficient, so this minimization-based approach is not
recommended.

\subsubsection{Center of mass-based}\label{app:SVM:bias:centroid}

Essentially the same center of mass-based approach as was described in
\appref{cutting-plane:centroid} can be used in this setting, except that we
must find the center of mass of a $2$-dimensional sublevel epigraph, rather
than an $m+1$-dimensional superlevel hypograph:
\begin{defn}
  \label{def:centroid-bias}
  \textbf{(Center of mass-based \algref{cutting-plane-bias})} \emph{Do}
  increase $u_0'$ on line 2, have CutChooser take $(b_t, z_t)$ to be the center
  of mass of $\left\{ \left( b, z \right) \mid h_t'\left(b\right) \le z \le
  U_t' \right\}$, and choose $\epsilon_t' = (U_t' - z_t) / 2$.
\end{defn}
Unlike in \appref{cutting-plane:maximization}, the fact that this problem is
one-dimensional enables us to efficiently implement this CutChooser by
explicitly representing each $h_t'$ as a set of piecewise linear segments, over
which computing an integral (and therefore the center of mass) is
straightforward, with a runtime that is linear in the number of segments.

Due to the similarity between
\algrefs{cutting-plane-bias}{cutting-plane-multipliers}, we can simply recycle
the results of \appref{cutting-plane:centroid}, with the troublesome second
term in the bound of \lemref{centroid-multipliers-convergence-rate} removed by
combining the ``maybe'' portion of \algref{cutting-plane-bias} with the
Lipschitz continuity of $\Psi$ as a function of $b$:

\begin{lem}{centroid-bias-convergence-rate}
  In the context of \algref{cutting-plane-bias}, suppose that we choose $b_t$
  and $\epsilon_t'$ as in \defref{centroid-bias}. Then there is a iteration
  count $T_{\epsilon'}$ satisfying:
  \begin{equation*}
    T_{\epsilon'} = O\left( \ln \left( \frac{ L B \left( \tilde{u}_0' - l_0'
    \right) }{\epsilon'} \right) \right) \eqcomma
  \end{equation*}
  such that, if $t \ge T_{\epsilon'}$, then $U_t' - L_t' \le \epsilon'$, where
  $B$ is the length of $\mathcal{B}$ and $L$ is as in
  \eqref{Lipschitz-definition}. Hence, \algref{cutting-plane-bias} will
  terminate after $T_{\epsilon'}$ iterations.
\end{lem}
\begin{proof}
  Starting from (and adapting) the final equation in the proof of
  \lemref{centroid-multipliers-convergence-rate}:
  \begin{align*}
    U_t' - L_t' \le & 4 \left( 1 - \frac{ 1 }{ 2e } \right)^{ t - 1 } \left(
    \frac{ 1 }{ \epsilon' } \right) \left( u_0' - l_0' \right)^2 \\
    & \cdot \left( \frac{ B }{ \measureof{\lebesguemeasure}{ \left\{ b \in
    \mathcal{B} \mid \min_{w \in \R^d} \Psi\left( w, b, \multipliers; w', b'
    \right) \le u_0' \right\} } } \right) \eqperiod
  \end{align*}
  Observe that, as a function of $b$, $\Psi\left( w, b, \multipliers; w', b'
  \right)$ is $L$-Lipschitz. Hence, if we let $w^* \in \R^d, b^* \in
  \mathcal{B}$ be the optimal weight and bias, then:
  \begin{equation*}
    \measureof{\lebesguemeasure}{ \left\{ b \in \mathcal{B} \mid \Psi\left(
    w^*, b, \multipliers; w', b' \right) \le u_0' \right\} } \ge \min\left\{ B,
    \frac{ u_0' - \Psi\left( w^*, b^*, \multipliers; w', b' \right) }{L}
    \right\} \eqperiod
  \end{equation*}
  Since $\min_{w \in \R^d}\Psi\left( w, b, \multipliers; w', b' \right) \le
  \Psi\left( w^*, b, \multipliers; w', b' \right)$, it follows that:
  \begin{equation*}
    U_t' - L_t' \le 4 \left( 1 - \frac{ 1 }{ 2e } \right)^{ t - 1 } \left(
    \frac{ 1 }{ \epsilon' } \right) \left( u_0' - l_0' \right)^2 \max\left\{ 1,
    \frac{ L B }{ u_0' - \Psi\left( w^*, b^*, \multipliers; w', b' \right) }
    \right\} \eqperiod
  \end{equation*}
  This is the reason that we increased $u_0'$ on line 2 of
  \algref{cutting-plane-bias}, since doing so has the result that $u_0' -
  \Psi\left( w^*, b^*, \multipliers; w', b' \right) \ge \tilde{u}_0' - l_0'$.
  Since we also have that $u_0' - l_0' = 2(\tilde{u}_0' - l_0')$:
  \begin{equation*}
    U_t' - L_t' \le 16 \left( 1 - \frac{ 1 }{ 2e } \right)^{ t - 1 } \left(
    \frac{ 1 }{ \epsilon' } \right) \left( \tilde{u}_0' - l_0' \right)
    \max\left\{ \tilde{u}_0' - l_0', L B \right\} \eqperiod
  \end{equation*}
  The same reasoning as was used in the proof of
  \lemref{centroid-multipliers-convergence-rate} then gives the claimed bound
  on $T_{\epsilon'}$.
\end{proof}

In addition to the above result, the obvious analogue of
\lemref{centroid-multipliers-epsilon-bound} holds as well:

\begin{lem}{centroid-bias-epsilon-bound}
  In the context of \algref{cutting-plane-bias}, suppose that we choose $b_t$
  and $\epsilon_t'$ as in \defref{centroid-bias}. Then:
  \begin{equation*}
    \epsilon_t' \ge \frac{ U_t' - L_t' }{ 2e } \eqcomma
  \end{equation*}
  and in particular, for all $t$ (before termination):
  \begin{equation*}
    \epsilon_t' \ge \frac{ \epsilon' }{ 2e } \eqcomma
  \end{equation*}
  since we terminate as soon as $U_t' - L_t' \le \epsilon'$.
\end{lem}
\begin{proof}
  Same as \lemref{centroid-multipliers-epsilon-bound}.
\end{proof}

In \appref{overall}, we'll combine these results with those of
\apprefs{SVM:SDCA}{cutting-plane} to bound the overall convergence rate of
\algref{cutting-plane-multipliers}.

\subsection{Kernelization}\label{app:SVM:kernel}

The foregoing discussion covers the case in which we wish to learn a linear
classifier, and use an SVM optimizer (SDCA) that doesn't handle an
unregularized bias. It's clear that we could freely substitute another linear
SVM optimizer for SDCA, as long as it finds both a primal and dual solution so
that we can calculate the lower and upper bounds required by
\algref{cutting-plane-multipliers}.

Our technique is easily kernelized---the resulting algorithm simply depends on
inner kernel SVM optimizations, rather than linear SVM optimizations. SDCA can
be used in the kernel setting, but the per-iteration cost increases from $O(d)$
arithmetic operations to $O(n)$ kernel evaluations, where $n$ is the total size
of all of the datasets.
Kernel-specific optimizers, such as \texttt{LIBSVM}~\citep{LibSVM}, will
generally work better than SDCA in practice, since they typically have the same
per-iteration cost, but each iteration is ``smarter''. More importantly, such
optimizers usually jointly optimize over $w$ and $b$, eliminating the need for
\algref{cutting-plane-bias} entirely---in other words, these algorithms could
be used to implement the higher-level $\mbox{SVMOptimizer}$, instead of the
lower-level $\mbox{SDCAOptimizer}$.
For this reason, an implementation based on such an optimizer is the simplest
version of our proposed approach.

%% file: figures/alg-cutting-plane-bias.tex
\begin{algorithm*}[t]

\begin{pseudocode}
\codename $\mbox{SVMOptimizer}\left( \multipliers, \tilde{u}_0', \epsilon' \right)$ \\
\codeline Initialize $g_0' \in \R$ to zero and $l_0' = -\sum_{j=1}^m \multipliers_j \svmconstraintbound{j}$ \\
\codeline \emph{Maybe} set $u_0' = 2 \tilde{u}_0' - l_0'$, otherwise $u_0' = \tilde{u}_0'$ \codecomment{needed for \lemref{centroid-bias-convergence-rate}} \\
\codeline For $t \in \{1, 2, \dots\}$ \\
\codeline \>Let $h_t'\left(b\right) = \max_{s\in\{ 0, 1, \dots, t - 1 \}} \left( l_s' + g_s'\left(b - b_s\right) \right)$ \\
\codeline \>Let $L_t' = \min_{b \in \mathcal{B}} h_t'\left( b \right)$ and $U_t' = \min_{s \in \{ 0, 1, \dots, t - 1 \} } u_s'$ \\
\codeline \>If $U_t' - L_t' \le \epsilon'$ then \\
\codeline \>\>Let $s \in \{ 1, \dots, t - 1\}$ be an index minimizing $u_s'$ \\
\codeline \>\>Return $w^{(s)}$, $b_s$, $L_t'$ \\
%\codeline \>\emph{Somehow} choose a $b_t \in \mathcal{B}$ for which $h_t'\left( b_t \right) \le U_t'$, and an $\epsilon_t' > 0$ \\
\codeline \>Let $b_t, \epsilon_t' = \mbox{CutChooser}\left(h_t', U_t'\right)$ \\
\codeline \>Let $\dualvariables^{(t)}, w^{(t)} = \mbox{SDCAOptimizer}\left( b_t, \multipliers, \epsilon_t' \right)$ \\
\codeline \>Let $u_t' = \Psi( w^{(t)}, b_t, \multipliers; w', b')$ \\
\codeline \>Let $l_t' = \Psi^*( \dualvariables^{(t)}, b_t, \multipliers; w', b')$ and $g_t' = \frac{\partial}{\partial_b'} \Psi^*( \dualvariables^{(t)}, b_t, \multipliers; w', b')$
\end{pseudocode}

\caption{
  Skeleton of a cutting-plane algorithm that finds a $b \in \mathcal{B}$
  minimizing (to within $\epsilon$) $\min_{b \in \mathcal{B}, w \in \R^d}
  \Psi(w, b, \multipliers; w', b')$, where $\mathcal{B} \subseteq \R$ is a
  closed interval.
  It is assumed that $\tilde{u}_0' \in \R$ is a finite upper bound on $\min_{b \in
  \mathcal{B}, w \in \R^d} \Psi(w, b, \multipliers; w', b')$, while by the
  definition of $\Psi$ (\eqref{psi-definition}), the $l_0'$ chosen on line 1
  will lower bound the same quantity.
  The $u_0'$ increase that is ``maybe'' performed on line 2, and the CutChooser
  function on line 9, are discussed in \appref{SVM:bias}.
  The SDCAOptimizer function is as described in \appref{SVM:SDCA}.
}

\label{alg:cutting-plane-bias}

\end{algorithm*}

%% file: app-overall.tex
\section{Overall convergence rates}\label{app:overall}

We may now combine the results in \apprefs{cutting-plane}{SVM} into one bound
on the overall convergence rate of \algref{cutting-plane-multipliers}, assuming
that we use \algref{cutting-plane-bias}, rather than an off-the-shelf SVM
solver, to implement the SVMOptimizer:

\begin{thm}{centroid-convergence-rate}
  Suppose that we take $l_0 = -\sum_{j=1}^m \multipliers_j \constraintbound{j}$
  in \algref{cutting-plane-multipliers}, that SVMOptimizer is implemented as in
  \algref{cutting-plane-bias}, and that the CutChooser functions in
  \algrefs{cutting-plane-multipliers}{cutting-plane-bias} are implemented using
  the center of mass (as in \defrefs{centroid-multipliers}{centroid-bias}).
  Then \algref{cutting-plane-multipliers} will perform:
  \begin{equation*}
    O\left( m \ln \left( \frac{u_0 - l_0}{\epsilon} \right) + \ln \left( \frac{
      \measureof{\lebesguemeasure}{\multiplierspace} }{
        \measureof{\lebesguemeasure}{ \superlevelset\left( \svmproblem, l_0
        \right) } } \right) \right)
  \end{equation*}
  iterations, each of which contains a single call to
  \algref{cutting-plane-bias}, with each such call requiring:
  \begin{equation*}
    O\left( \ln \left( \frac{ L B m \left( u_0 - l_0 \right) }{\epsilon}
    \right) \right)
  \end{equation*}
  iterations, each of which contains a single call to SDCAOptimizer, with each
  such call requiring:
  \begin{equation*}
    O\left( \max\left\{ 0, n \ln \left( \frac{\lambda n}{L^2 X^2} \right)
    \right\} + n + \frac{L^2 X^2 m}{\lambda \epsilon} \right)
  \end{equation*}
  iterations, each of which requires $O(d)$ arithmetic operations.
\end{thm}
\begin{proof}
  Notice that $u_0 \ge \tilde{u}_0' \ge l_0' \ge l_0$ (the first inequality
  because \algref{cutting-plane-multipliers} passes a quantity upper bounded by
  $u_0$ to SVMOptimizer, and the second by our choice of $l_0$).  By
  \lemref{centroid-multipliers-epsilon-bound}, we also have that $\epsilon' \ge
  \epsilon / 2e (m + 1)$. The claimed results follow immediately from these
  facts, combined with
  \lemrefsss{centroid-multipliers-convergence-rate}{centroid-multipliers-epsilon-bound}{centroid-bias-convergence-rate}{centroid-bias-epsilon-bound}
  and \thmref{SDCA}.
\end{proof}

We can simplify (or perhaps \emph{over}simplify) this result by considering
only the total number of training examples $n$, number of constraints $m$,
number of datasets $k$, dimension $d$ and desired suboptimality $\epsilon$,
dropping all of the other factors, and assuming that the sizes of the $k$
datasets differ only by a constant factor (so that, as explained in
\appref{SVM:SDCA}, we can take the Lipschitz constant $L$ to be of order $k$).
Then the overall cost of finding an $\epsilon$-suboptimal solution to
\probref{convex-objective} will be $\tilde{O}\left( d n m + d m^2 k^2 /
\epsilon \right)$ total arithmetic operations in the inner SDCA optimizers,
plus $O\left( m \ln^2\left( k / \epsilon \right) \right)$ calls to the center
of mass oracles in \algrefs{cutting-plane-multipliers}{cutting-plane-bias}, and
another $O\left( m \ln^2 \left(k / \epsilon \right) \right)$ calls to a linear
programming oracle for finding $U_t$ in \algref{cutting-plane-multipliers} and
$L_t$ in \algref{cutting-plane-bias}.
%
%\NOTE{\citep[Chapter 10.1]{Nemirovski:2004} gives that a $d$-dimensional LP
%with $m$ constraints can be solved in $d^2 m$ time}

We must reiterate that, as we mentioned in \appref{cutting-plane:centroid},
finding the center of mass is a computationally difficult problem. Hence, our
reliance on a center of mass oracle for the optimization over $\multipliers$ is
unrealistic (there is no problem when optimizing over $b$, since the underlying
problem is one-dimensional). With that said, we hope that these results can
provide a basis for future work.

%% file: main.bbl
\begin{thebibliography}{27}
\providecommand{\natexlab}[1]{#1}
\providecommand{\url}[1]{\texttt{#1}}
\expandafter\ifx\csname urlstyle\endcsname\relax
  \providecommand{\doi}[1]{doi: #1}\else
  \providecommand{\doi}{doi: \begingroup \urlstyle{rm}\Url}\fi

\bibitem[Ball(1997)]{Ball:1997}
K.~Ball.
\newblock An elementary introduction to modern convex geometry.
\newblock \emph{Flavors of Geometry}, 31:\penalty0 1--58, 1997.

\bibitem[Bartlett and Mendelson(2002)]{Bartlett:2002}
P.~L. Bartlett and S.~Mendelson.
\newblock Rademacher and {G}aussian complexities: Risk bounds and structural
  results.
\newblock \emph{JMLR}, 3:\penalty0 463--482, 2002.

\bibitem[Biddle(2005)]{Biddle:2005}
D.~Biddle.
\newblock \emph{Adverse Impact and Test Validation: A Practitioner's Guide to
  Valid and Defensible Employment Testing}.
\newblock Gower, 2005.

\bibitem[Bland et~al.(1981)Bland, Goldfarb, and Todd]{Bland:1981}
R.~G. Bland, D.~Goldfarb, and M.~J. Todd.
\newblock Feature article---the ellipsoid method: A survey.
\newblock \emph{Operations Research}, 29\penalty0 (6):\penalty0 1039--1091,
  November 1981.

\bibitem[Boyd and Vandenberghe(2011)]{Boyd:2011}
S.~Boyd and L.~Vandenberghe.
\newblock Localization and cutting-plane methods, April 2011.
\newblock Stanford EE 364b lecture notes.

\bibitem[Chang and Lin(2011)]{LibSVM}
C.-C. Chang and C.-J. Lin.
\newblock {LIBSVM}: A library for support vector machines.
\newblock \emph{ACM Transactions on Intelligent Systems and Technology},
  2:\penalty0 27:1--27:27, 2011.
\newblock Software available at \url{http://www.csie.ntu.edu.tw/~cjlin/libsvm}.

\bibitem[Collobert et~al.(2006)Collobert, Sinz, Weston, and
  Bottou]{Collobert:2006}
R.~Collobert, F.~Sinz, J.~Weston, and L.~Bottou.
\newblock Trading convexity for scalability.
\newblock In \emph{ICML}, 2006.

\bibitem[Cotter et~al.(2013)Cotter, Shalev-Shwartz, and Srebro]{Cotter:2013}
A.~Cotter, S.~Shalev-Shwartz, and N.~Srebro.
\newblock Learning optimally sparse support vector machines.
\newblock In \emph{ICML}, pages 266--274, 2013.

\bibitem[Davenport et~al.(2010)Davenport, Baraniuk, and Scott]{Davenport:2010}
M.~Davenport, R.~G. Baraniuk, and C.~D. Scott.
\newblock Tuning support vector machines for minimax and {Neyman-Pearson}
  classification.
\newblock \emph{IEEE Transactions on Pattern Analysis and Machine
  Intelligence}, 2010.

\bibitem[Eban et~al.(2016)Eban, Schain, Gordon, Saurous, and Elidan]{Eban:2016}
E.~E. Eban, M.~Schain, A.~Gordon, R.~A. Saurous, and G.~Elidan.
\newblock Large-scale learning with global non-decomposable objectives, 2016.
\newblock URL \url{https://arxiv.org/abs/1608.04802}.

\bibitem[Fan et~al.(2008)Fan, Chang, Hsieh, Wang, and Lin]{LibLinear}
R.-E. Fan, K.-W. Chang, C.-J. Hsieh, X.-R. Wang, and C.-J. Lin.
\newblock {LIBLINEAR}: A library for large linear classification.
\newblock \emph{JMLR}, 9:\penalty0 1871--1874, 2008.

\bibitem[Fard et~al.(2016)Fard, Cormier, Canini, and Gupta]{Fard:2016}
M.~M. Fard, Q.~Cormier, K.~Canini, and M.~Gupta.
\newblock Launch and iterate: Reducing prediction churn.
\newblock In \emph{NIPS}, 2016.

\bibitem[Gasso et~al.(2011)Gasso, Pappaionannou, Spivak, and
  Bottou]{Bottou:2011}
G.~Gasso, A.~Pappaionannou, M.~Spivak, and L.~Bottou.
\newblock Batch and online learning algorithms for nonconvex {Neyman-Pearson}
  classification.
\newblock \emph{ACM Transactions on Intelligent Systems and Technology}, 2011.

\bibitem[Gr\"{u}nbaum(1960)]{Grunbaum:1960}
B.~Gr\"{u}nbaum.
\newblock Partitions of mass-distributions and convex bodies by hyperplanes.
\newblock \emph{Pacific Journal of Mathematics}, 10\penalty0 (4):\penalty0
  1257--1261, December 1960.

\bibitem[Hardt et~al.(2016)Hardt, Price, and Srebro]{Hardt:2016}
M.~Hardt, E.~Price, and N.~Srebro.
\newblock Equality of opportunity in supervised learning.
\newblock In \emph{NIPS}, 2016.

\bibitem[Joachims(2005)]{Joachims:2005b}
T.~Joachims.
\newblock A support vector method for multivariate performance measures.
\newblock In \emph{ICML}, 2005.

\bibitem[Mann and {McCallum}(2007)]{Mann:2007}
G.~S. Mann and A.~{McCallum}.
\newblock Simple, robust, scalable semi-supervised learning with expectation
  regularization.
\newblock In \emph{ICML}, 2007.

\bibitem[Miettinen(2012)]{Miettinen:2012}
K.~Miettinen.
\newblock \emph{Nonlinear multiobjective optimization}, volume~12.
\newblock Springer Science \& Business Media, 2012.

\bibitem[Narasimhan et~al.(2015)Narasimhan, Kar, and Jain]{Narasimhan:2015}
H.~Narasimhan, P.~Kar, and P.~Jain.
\newblock Optimizing non-decomposable performance measures: a tale of two
  classes.
\newblock In \emph{ICML}, 2015.

\bibitem[Nemirovski(1994)]{Nemirovski:1994}
A.~Nemirovski.
\newblock Lecture notes: Efficient methods in convex programming.
\newblock 1994.
\newblock URL \url{http://www2.isye.gatech.edu/~nemirovs/Lect_EMCO.pdf}.

\bibitem[Rademacher(2007)]{Rademacher:2007}
L.~Rademacher.
\newblock Approximating the centroid is hard.
\newblock In \emph{SoCG}, pages 302--305, 2007.

\bibitem[Rockafellar and Uryasev(2000)]{Rockafellar:2000}
R.~T. Rockafellar and S.~Uryasev.
\newblock Optimization of conditional value-at-risk.
\newblock \emph{Journal of Risk}, 2:\penalty0 21--42, 2000.

\bibitem[Scott and Nowak(2005)]{Scott:2005}
C.~D. Scott and R.~D. Nowak.
\newblock A {Neyman-Pearson} approach to statistical learning.
\newblock \emph{{IEEE} Transactions on Information Theory}, 2005.

\bibitem[Shalev-Shwartz and Zhang(2013)]{ShalevShwartz:2013}
S.~Shalev-Shwartz and T.~Zhang.
\newblock Stochastic dual coordinate ascent methods for regularized loss.
\newblock \emph{JMLR}, 14\penalty0 (1):\penalty0 567--599, Feb. 2013.

\bibitem[Shalev-Shwartz et~al.(2011)Shalev-Shwartz, Singer, Srebro, and
  Cotter]{ShalevShwartz:2010}
S.~Shalev-Shwartz, Y.~Singer, N.~Srebro, and A.~Cotter.
\newblock Pegasos: {P}rimal {E}stimated sub-{G}r{A}dient {SO}lver for {SVM}.
\newblock \emph{Mathematical Programming}, 127\penalty0 (1):\penalty0 3--30,
  March 2011.

\bibitem[Vuolo and Levy(2013)]{Vuolo:2013}
M.~S. Vuolo and N.~B. Levy.
\newblock Disparate impact doctrine in fair housing.
\newblock \emph{New York Law Journal}, 2013.

\bibitem[Zafar et~al.(2015)Zafar, Valera, Rodriguez, and Gummadi]{Zafar:2015}
M.~B. Zafar, I.~Valera, M.~G. Rodriguez, and K.~P. Gummadi.
\newblock Fairness constraints: A mechanism for fair classification.
\newblock In \emph{ICML Workshop on Fairness, Accountability, and Transparency
  in Machine Learning}, 2015.

\end{thebibliography}
